\documentclass[jair,twoside,11pt]{article}
\input epsf

\usepackage{jair}
\usepackage{theapa}
\usepackage{times}
\usepackage{helvet}
\usepackage{courier}
\usepackage{latexsym}
\usepackage{amssymb}
\usepackage{amsmath}
\usepackage{amsthm}
\usepackage{url}
\usepackage{graphicx}
\usepackage{subfigure}
\newtheorem{definition}{Definition}
\newtheorem{proposition}{Proposition}
\newtheorem{lemma}{Lemma}

\newtheorem{theorem}{Theorem}
\newtheorem{corollary}[theorem]{Corollary}

\newtheorem{acknowledgement*}{Acknowledgment}

\newcommand{\eoe}{\bigtriangleup}
\newcommand{\bdf}{\beta}
\newcommand{\comp}{\mathit{Comp}}
\newcommand{\cmp}{\mathit{\ comp\ }}
\newcommand{\plcc}{\mathit{PL^{cc}}}
\newcommand{\plmc}{\mathit{PL^{mc}}}
\newcommand{\At}{\mathit{At}}
\newcommand{\dom}{\mathit{Dom}}
\newcommand{\hd}{\mathit{hd}}
\newcommand{\bd}{\mathit{bd}}
\newcommand{\hs}{\mathit{hset}}

\newcommand{\n}{\mathbf{not}}

\newcommand{\smd}{\mathit{smodels}}

\newcommand{\cmd}{\mathit{cmodels}}

\newcommand{\pmd}{\mathit{pbmodels}}
\newcommand{\satzoo}{\mathit{satzoo}}
\newcommand{\pbs}{\mathit{pbs}}

\newcommand{\wsatcc}{\mathit{wsatcc}}
\newcommand{\wsatoip}{\mathit{wsatoip}}
\newcommand{\pmdzoo}{\mathit{pbmodels}\mbox{-}\mathit{satzoo}}

\newcommand{\pmdpbs}{\mathit{pbmodels}\mbox{-}\mathit{pbs}}
\newcommand{\pmdwcc}{\mathit{pbmodels}\mbox{-}\mathit{wsatcc}}
\newcommand{\pmdoip}{\mathit{pbmodels}\mbox{-}\mathit{wsatoip}}

\newcommand{\plwc}{\mathit{PL^{wa}}}
\newcommand{\pb}{\mathit{PB}}
\newcommand{\tocl}{\mathit{\tau_{cl}}}
\newcommand{\topb}{\mathit{\tau_{pb}}}

\jairheading{27}{2006}{299-334}{01/06}{11/06}
\ShortHeadings{Properties and Applications of Programs with Monotone and 
Convex Constraints}
{Liu \& Truszczy\'nski}
\firstpageno{299}

\begin{document}

\date{}

\title{Properties and Applications of Programs with Monotone and Convex 
Constraints}

\author{\name Lengning Liu \email lliu1@cs.uky.edu\\
\name Miros\l aw Truszczy\'nski \email mirek@cs.uky.edu\\
\addr Department of Computer Science, University of Kentucky,\\
Lexington, KY 40506-0046, USA}

\maketitle

\begin{abstract}
We study properties of programs with {\em monotone} and {\em convex} 
constraints. We extend to these formalisms concepts and results from 
normal logic programming. They include the notions of strong and
uniform equivalence with their characterizations, tight programs and 
Fages Lemma, program completion and loop formulas. Our results provide
an abstract account of properties of some recent extensions of logic 
programming with aggregates, especially the formalism of {\em lparse}
programs. They imply a method to compute stable models of {\em lparse} 
programs by means of off-the-shelf solvers of pseudo-boolean 
constraints, which is often much faster than the {\em smodels} system.

\end{abstract}

\section{Introduction}

We study programs with {\em monotone} constraints \cite{mt04,mnt03,mnt06} 
and introduce a related class of programs with {\em convex} constraints. 
These formalisms allow constraints to appear in the heads of program 
rules, which sets them apart from other recent proposals for integrating 
constraints into logic programs \cite{p04,pdb04,pdb06,dlv-agg-03,flp04},
and makes them suitable as an abstract basis for formalisms such as {\em
lparse} programs \cite{sns02}. 

We show that several results from normal logic programming generalize to 
programs with monotone constraints. We also discuss how these techniques 
and results can be extended further to the setting of programs with 
convex constraints. We then apply some of our general results to design 
and implement a method to compute stable models of {\em lparse} programs 
and show that it is often much more effective than {\em smodels} 
\cite{sns02}.

Normal logic programming with the semantics of stable models is an 
effective knowledge representation formalism, mostly due to its ability 
to express default assumptions \cite{baral03,GelLeo02}. However, 
modeling numeric constraints on sets in normal logic programming is 
cumbersome, requires auxiliary atoms and leads to large programs
hard to process efficiently. Since such constraints, often called {\em 
aggregates}, are ubiquitous, researchers proposed extensions of normal 
logic programming with explicit means to express aggregates, and 
generalized the stable-model semantics to the extended settings.

Aggregates imposing bounds on weights of sets of atoms and literals,
called {\em weight} constraints, are especially common in practical 
applications and are included in all recent extensions of logic programs 
with aggregates. Typically, these extensions do not allow aggregates to 
appear in the heads of rules. A notable exception is the formalism of 
{\em programs with weight constraints} \cite{nss99,sns02},
which we refer to as {\em lparse} programs\footnote{Aggregates in the 
heads of rules have also been studied recently by \citeA{sp06} and 
\citeA{spt06}.}.

{\em Lparse} programs are logic programs whose rules have weight
constraints in their heads and whose bodies are conjunctions of weight
constraints. Normal logic programs can be viewed as a subclass of {\em 
lparse} programs and the semantics of {\em lparse} programs generalizes 
the stable-model semantics of
normal logic programs \cite{gl88}.
{\em Lparse} programs are 
one of the most commonly used extensions of logic programming with 
weight constraints.

Since rules in {\em lparse} programs may have weight constraints as 
their heads, the concept of one-step provability is nondeterministic, 
which hides direct parallels between {\em lparse} and normal logic 
programs. An explicit connection emerged when
\citeA{mt04} and \citeA{mnt03,mnt06}
introduced {\em logic programs with monotone constraints}. These programs 
allow aggregates in the heads of rules and support nondeterministic 
computations.
\citeA{mt04} and \citeA{mnt03,mnt06}
proposed a 
generalization of the van Emden-Kowalski one-step provability operator 
to account for that nondeterminism, defined supported and stable models 
for programs with monotone constraints that mirror their normal logic 
programming counterparts, and showed encodings of {\em smodels} programs
as programs with monotone constraints.

In this paper, we continue investigations of programs with monotone
constraints. We show that the notions of uniform and strong equivalence
of programs \cite{lpv01,lin02,tu03,ef03} extend to programs with 
monotone constraints, and that their characterizations \cite{tu03,ef03}
generalize, too.

We adapt to programs with monotone constraints the notion 
of a {\em tight} program \cite{el03} and generalize Fages Lemma 
\cite{fag94}.

We introduce extensions of propositional logic with monotone 
constraints. We define the completion of a monotone-constraint
program with respect to this logic, and generalize the notion of a 
loop formula. We then prove the loop-formula characterization of stable 
models of programs with monotone constraints, extending to the setting 
of monotone-constraint programs results obtained for normal logic 
programs
by \citeA{cl78} and \citeA{lz02}.

Programs with monotone constraints make explicit references to the
default negation operator. We show that by allowing a more general class
of constraints, called {\em convex}, default negation can be eliminated 
from the language. We argue that all results in our paper extend to 
programs with convex constraints.

Our paper shows that programs with monotone and convex constraints have 
a rich theory that closely follows that of normal logic programming. 
It implies that programs with monotone and convex constraints form
an abstract generalization of extensions of normal logic programs. In 
particular, all results we obtain in the abstract setting of programs
with monotone and convex constraints specialize to {\em lparse} programs 
and, in most cases, yield results that are new.

These results have practical implications. The properties of the program 
completion and loop formulas, when specialized to the class of {\em 
lparse} programs, yield a method to compute stable models of {\em 
lparse} programs by means of solvers of {\em pseudo-boolean} constraints,
developed by the propositional satisfiability and integer programming
communities \cite{es03,arms02,wal97,pbcomp05,lt03}. We describe 
this method in detail and present experimental results on its performance. 
The results show that our method on problems we used for testing
typically outperforms {\em smodels}.

\section{Preliminaries}
\label{prel}

We consider the propositional case only and assume a fixed set $\At$
of propositional atoms. It does not lead to loss of generality, as it 
is common to interpret programs with variables in terms of their 
propositional groundings.

The definitions and results we present in this section come from
papers by \citeA{mt04} and \citeA{mnt06}.
Some of them 
are more general as we allow constraints with 
infinite domains and programs with inconsistent constraints in the 
heads.

\noindent
{\bf Constraints.} A {\em constraint} is an expression $A=(X,C)$, 
where $X\subseteq \At$ and $C\subseteq {\cal P}(X)$ (${\cal P}(X)$ 
denotes the powerset of $X$). We call the set $X$ the {\em domain} 
of the constraint $A=(X,C)$ and denote it by $\dom(A)$. Informally
speaking, a constraint $(X,C)$ describes a property of subsets of its 
domain, with $C$ consisting precisely of these subsets of $X$ that 
{\em satisfy} the constraint (have property) $C$.

In the paper, we identify truth assignments (interpretations) with the 
sets of atoms they assign the truth value {\em true}. That is, given an
interpretation $M\subseteq \At$, we have $M\models a$ if and only if $a
\in M$. We say that an interpretation $M \subseteq\At$ {\em satisfies} a 
constraint $A=(X,C)$ ($M\models A$), if $M\cap X \in C$. Otherwise, 
$M$ does not satisfy $A$, ($M\not \models A$).

A constraint $A=(X,C)$ is {\em consistent} if there is $M$ such
that $M\models A$. Clearly, a constraint $A=(X,C)$ is consistent if and
only if $C\not= \emptyset$.

We note that propositional atoms can be regarded as constraints. Let $a
\in \At$ and $M\subseteq \At$. We define $C(a)=(\{a\},\{\{a\}\})$. It is
evident that $M\models C(a)$ if and only if $M\models a$. Therefore, in 
the paper we often write $a$ as a shorthand for the constraint $C(a)$.

\noindent
{\bf Constraint programs.} Constraints are building blocks of rules and 
programs.
\citeA{mt04}
defined {\em constraint programs} as sets of {\em 
constraint} rules 
\begin{equation}
\label{eq1a}
A \leftarrow A_1, \ldots, A_k, \n(A_{k+1}),\ldots,\n(A_m)
\end{equation}
where $A$, $A_1,\ldots,A_n$ are constraints and $\n$ is the {\em default 
negation} operator. 

In the context of constraint programs, we refer to constraints and
negated constraints as {\em literals}. Given a rule $r$ of the form
(\ref{eq1a}), the constraint (literal) $A$ is the {\em head} of $r$
and the set $\{A_1,\ldots,$ $A_k,\ldots,\n(A_{k+1}),\ldots,\n(A_m)\}$ 
of literals is the {\em body} of $r$\footnote{Sometimes we view the 
body of a rule as the {\em conjunction} of its literals.}. We denote 
the head and the body of $r$ by $\hd(r)$ and $\bd(r)$, respectively. 
We define the the {\em headset} of $r$, written $\hs(r)$, as the domain 
of the head of $r$. That is, $\hs(r)=\dom(\hd(r))$.

For a constraint program $P$, we denote by $\At(P)$ the set of atoms 
that appear in the domains of constraints in $P$. We define the {\em 
headset} of $P$, written $\hs(P)$, as the union of the headsets of all 
rules in $P$. 

\noindent
{\bf Models.}
The concept of satisfiability extends in a standard way to literals 
$\n(A)$ ($M\models \n(A)$ if $M\not\models A$), to sets (conjunctions) 
of literals and, finally, to constraint programs.

\noindent
{\bfseries {\slshape M}-applicable rules.} Let $M\subseteq \At$ be an 
interpretation. A rule (\ref{eq1a}) is {\em $M$-applicable} if $M$ 
satisfies every literal in $\bd(r)$. We denote by $P(M)$ the set of 
all $M$-applicable rules in $P$.

\noindent
{\bf Supported models.}
Supportedness is a property of models. Intuitively, every atom $a$ in a 
supported model must have ``reasons'' for being ``in''. Such reasons 
are $M$-applicable rules whose heads contain $a$ in their domains. 
Formally, let $P$ be a constraint program and $M$ a subset of $\At(P)$. 
A model $M$ of $P$ is {\em supported} if $M\subseteq \hs(P(M))$.

\noindent
{\bf Examples.}
  We illustrate the concept with examples. Let $P$ be the constraint 
  program that consists of the following two rules:
  \begin{quote}
  $(\{c, d, e\}, \{\{c\}, \{d\}, \{e\}, \{c,d,e\}\})\leftarrow $\\
  $(\{a, b\}, \{\{a\}, \{b\}\}) \leftarrow (\{c, d\},\{\{c\}, \{c, 
  d\}\}), \n((\{e\}, \{\{e\}\}))$
  \end{quote}

  A set $M=\{a,c\}$ is a model of $P$ as $M$ satisfies the heads of the
  two rules. Both rules in $P$ are $M$-applicable. The first of them
  provides the support for $c$, the second one --- for $a$. Thus, $M$ is
  a supported model.

  A set $M'=\{a, c, d, e\}$ is also a model of $P$. However, $a$ has no
  support in $P$. Indeed, $a$ only appears in the headset of the second 
  rule. This rule is not $M'$-applicable and so, it does not support
  $a$. Therefore, $M'$ is not a supported model of $P$.
  \hfill$\eoe$
  
\noindent
{\bf Nondeterministic one-step provability.} Let $P$ be a constraint 
program and $M$ a set of atoms. A set $M'$ is {\em nondeterministically 
one-step provable} from $M$ by means of $P$, if $M'\subseteq \hs(P(M))$ 
and $M' \models \hd(r)$, for every rule $r$ in $P(M)$.

The {\em nondeterministic one-step provability operator} $T_P^{nd}$ for 
a program $P$ is an operator on ${\cal P}(\At)$ such that for every $M 
\subseteq \At$, $T_P^{nd}(M)$ consists of all sets that are 
nondeterministically one-step provable from $M$ by means of $P$.

The operator $T_P^{nd}$ is {\em nondeterministic} as it assigns to each 
$M\subseteq \At$ a {\em family} of subsets of $\At$, each being a 
possible outcome of applying $P$ to $M$. In general, $T_P^{nd}$ is {\em 
partial}, since there may be sets $M$ such that $T_P^{nd}(M)=\emptyset$ 
(no set can be derived from $M$ by means of $P$). For instance, if 
$P(M)$ contains a rule $r$ such that $\hd(r)$ is inconsistent, then 
$T_P^{nd}(M)=\emptyset$.

\noindent
{\bf Monotone 
constraints.} A constraint $(X,C)$ is 
{\em monotone} if $C$ is closed under superset, that is, for every $W, 
Y \subseteq X$, if $W \in C$ and $W \subseteq Y$ then $Y\in C$. 

Cardinality and weight constraints provide examples of monotone
constraints. Let $X$ be a {\em finite} set and let $C_k(X)=\{Y\colon 
Y \subseteq X,\ k\leq |Y|\}$, where $k$ is a non-negative integer. 
Then $(X,C_k(X))$ is a constraint expressing the property that a 
subset of $X$ has at least $k$ elements. We call it a {\em lower-bound 
cardinality constraint} on $X$ and denote it by $kX$. 

A more general class of constraints are {\em weight constraints}. Let 
$X$ be a finite set, say $X=\{x_1,\ldots,x_n\}$, and let $w, w_1,\ldots,
w_n$ be non-negative reals. We interpret each $w_i$ as the {\em weight} 
assigned to $x_i$. A {\em lower-bound weight constraint} is a constraint 
of the form $(X,C_w)$, where $C_w$ consists of those subsets of $X$ 
whose total weight (the sum of weights of elements in the subset) is at 
least $w$. We write it as
\[
w[x_1=w_1,\ldots, x_n=w_n].
\]

If all weights are equal to 1 and $w$ is an integer, weight constraints 
become cardinality constraints. We also note that the constraint $C(a)$
 is a cardinality constraint $1\{a\}$ and also a weight constraint 
$1[a=1]$. Finally, we observe that lower-bound cardinality and weight 
constraints are monotone.  


Cardinality and weight constraints (in a somewhat more general form)
appear in the language of {\em lparse} programs \cite{sns02}, which
we discuss later in the paper. The 
notation we adopted for these constraints in this paper follows
the one proposed by \citeA{sns02}. 

We use cardinality and weight constraints in some of our examples. 
They are also the focus of the last part of the paper, where we use 
our abstract results to design a new algorithm to compute models of 
{\em lparse} programs. 

\noindent
{\bf Monotone-constraint programs.} We call constraint programs built 
of monotone constraints  --- {\em monotone-constraint programs} or 
{\em programs with monotone constraints}. That is, monotone-constraint 
programs consist of rules of rules of the form (\ref{eq1a}),
where $A$, $A_1,\ldots,A_m$ are {\em monotone} constraints. 

From now on, unless explicitly stated otherwise, programs we consider
are monotone-constraint programs.

\subsection{Horn Programs and Bottom-up Computations}

Since we allow constraints with infinite domains and inconsistent 
constraints in heads of rules, the results given in this subsection are 
more general than their counterparts
by \citeA{mt04} and \citeA{mnt03,mnt06}.
Thus, for the
sake of completeness, we present them with proofs.

A rule (\ref{eq1a}) is {\em Horn} if $k=m$ (no occurrences of the
negation operator in the body or, equivalently, only monotone 
constraints). A constraint program is {\em Horn} if every rule in the 
program is Horn.

With a Horn constraint program we associate {\em bottom-up} computations, 
generalizing the corresponding notion of a bottom-up computation for
a normal Horn program. 

\begin{definition}
\label{defPC}
Let $P$ be a Horn program. A {\em $P$-computation} is a (transfinite)
sequence $\langle X_\alpha\rangle$ such that
\begin{enumerate}
\item $X_0 = \emptyset$,
\item for every ordinal number $\alpha$, $X_\alpha\subseteq
X_{\alpha+1}$ and $X_{\alpha+1} \in T_P^{nd}(X_\alpha)$, 
\item for every {\em limit} ordinal $\alpha$, $X_{\alpha}
=\bigcup_{\beta<\alpha} X_\beta$.
\end{enumerate}
\end{definition}

Let $t=\langle X_\alpha \rangle$ be a $P$-computation. Since for every 
$\beta < \beta'$, $X_\beta\subseteq X_{\beta'} \subseteq \At$, 
there is a least ordinal number $\alpha_t$ such that $X_{\alpha_t
+1} = X_{\alpha_t}$, in other words, a least ordinal when the
$P$-computation stabilizes. We refer to $\alpha_t$ as the {\em length} of
the $P$-computation  $t$.

\noindent
{\bf Examples.}
Here is a simple example showing that some programs have computations
of length exceeding $\omega$ and so, the transfinite induction in the
definition cannot be avoided. Let $P$ be the program consisting of the 
following rules:

\begin{quote}
$(\{a_0\},\{\{a_0\}\}) \leftarrow .$\\
$(\{a_i\},\{\{a_i\}\}) \leftarrow (X_{i-1},\{X_{i-1}\})$, 
 for $i=1,2,\ldots$\\
$(\{a\},\{\{a\}\}) \leftarrow (X_\infty,\{X_\infty\})$,
\end{quote}
where $X_i=\{a_0,\ldots a_i\}$, $0\leq i$, and $X_\infty=\{a_0,a_1,
\ldots\}$.
Since the body of the last rule contains a constraint with an infinite
domain $X_\infty$, it does not become applicable in any finite step of 
computation. However, it does become applicable in the step $\omega$ and 
so, $a\in X_{\omega+1}$. Consequently, $X_{\omega+1}\not=X_\omega$.
\hfill$\eoe$

For a $P$-computation $t=\langle X_\alpha\rangle$, we call $\bigcup_
{\alpha} X_\alpha$ the {\em result} of the computation and denote it by 
$R_t$. Directly from the definitions, it follows that
$R_t=X_{\alpha_t}$.

\begin{proposition}
\label{propresmod}
Let $P$ be a Horn constraint program and $t$ a $P$-computation. 
Then $R_t$ is a supported model of $P$.
\end{proposition}
\begin{proof}
   Let $M=R_t$ be the result of a $P$-computation $t = \langle
   X_\alpha\rangle$. We need to show that: (1) $M$ is a model of $P$;
   and (2) $M\subseteq \hs(P(M))$.

   \noindent
   (1) Let us consider a rule $r\in P$ such that $M\models \bd(r)$. Since
   $M=R_t=X_{\alpha_t}$ (where $\alpha_t$ is the length of $t$),
   $X_{\alpha_t} \models \bd(r)$. Thus, $X_{\alpha_t+1}\models \hd(r)$.
   Since $M=X_{\alpha_t+1}$, $M$ is a model of $r$ and, consequently,
   of $P$, as well.

   \noindent
   (2) We will prove by induction that, for every set $X_\alpha$ in the
   computation $t$, $X_\alpha\subseteq \hs(P(M))$. The base case holds
   since $X_0 = \emptyset \subseteq \hs(P(M))$.

   If $\alpha=\beta+1$, then $X_\alpha \in T_P^{nd}(X_{\beta})$. It
   follows that $X_\alpha \subseteq \hs(P(X_{\beta}))$. Since $P$ is
   a Horn program and $X_{\beta}\subseteq M$, $\hs(P(X_{\beta}))
   \subseteq \hs(P(M))$. Therefore, $X_\alpha \subseteq \hs(P(M))$.

   If $\alpha$ is a limit ordinal, then $X_\alpha=\bigcup_{\beta<\alpha}
   X_\beta$. By the induction hypothesis, for every $\beta<\alpha$,
   $X_\beta\subseteq \hs(P(M))$. Thus, $X_\alpha\subseteq \hs(P(M))$.
   By induction, $M\subseteq \hs(P(M))$. 
\end{proof}

\noindent
{\bf Derivable models.} We use computations to define {\em derivable} 
models of Horn constraint programs. A set $M$ of atoms is a {\em 
derivable model} of a Horn constraint program $P$ if for some
$P$-computation $t$, we have $M=R_t$. By Proposition \ref{propresmod}, 
derivable models of $P$ are supported models of $P$ and so, also models 
of $P$.

  Derivable models are similar to the least model of a normal Horn
  program in that both can be derived from a program by means of a
  bottom-up computation. However, due to the nondeterminism of 
  bottom-up computations of Horn constraint programs, derivable models 
  are not in general unique nor minimal.
  
\noindent
{\bf Examples.}
  For example, let $P$ be the following Horn constraint program:
  \[
  P = \{ 1\{a, b\} \leftarrow \}
  \]
  Then $\{a\}$, $\{b\}$ and $\{a, b\}$ are its derivable models. The
  derivable models $\{a\}$ and $\{b\}$ are minimal models of $P$. 
  The third derivable model, $\{a, b\}$, is not a minimal model of $P$.
  \hfill$\eoe$
  
Since inconsistent monotone constraints may appear in the heads of Horn
rules, there are Horn programs $P$ and sets $X\subseteq \At$, such that 
$T_P^{nd}(X)=\emptyset$. Thus, some Horn constraint programs have no 
computations and no derivable models. However, if a Horn constraint 
program has models, the existence of computations and derivable models 
is guaranteed.

To see this, let $M$ be a model of a Horn constraint program $P$. We 
define a {\em canonical computation} $t^{P,M} = \langle X_\alpha^{P,M}
\rangle$ by specifying the choice of the next set in the computation 
in part (2) of Definition \ref{defPC}. Namely, for every ordinal 
$\beta$, we set 
\[
X_{\beta+1}^{P,M} = \hs(P(X_\beta^{P,M})) \cap M. 
\]
That is, we include in $X_{\alpha}^{P,M}$ {\em all} those atoms 
occurring in the heads of $X_\beta^{P,M}$-applicable rules that belong 
to $M$. We denote the result of $t^{P,M}$ by $Can(P,M)$. Canonical 
computations are indeed $P$-computations.

\begin{proposition}
\label{canIsComp}
Let $P$ be a Horn constraint program. If $M \subseteq \At$ is a 
model of $P$, the sequence $t^{P,M}$ is a $P\mbox{-}$computation.
\end{proposition}
\begin{proof}
As $P$ and $M$ are fixed, to simplify the notation in the proof we will 
write $X_\alpha$ instead of $X^{P,M}_\alpha$.

To prove the assertion, it suffices to show that
(1) $\hs(P(X_{\alpha})) \cap M \in T_P^{nd}(X_{\alpha})$,
and (2) $X_{\alpha} \subseteq \hs(P(X_{\alpha})) \cap M$, for every ordinal 
$\alpha$.

\noindent
(1) Let $X\subseteq M$ and $r\in P(X)$. Since all constraints in $\bd(r)$ 
are monotone, and $X\models \bd(r)$, $M\models \bd(r)$, as well. From 
the fact that $M$ is a model of $P$ it follows now that $M\models 
\hd(r)$. Consequently, $M\cap \hs(P(X)) \models \hd(r)$ for every $r\in 
P(X)$. Since $M\cap \hs(P(X))\subseteq \hs(P(X))$, 
\[
M\cap \hs(P(X)) \in T^{nd}_P(X).
\]
Directly from the definition of the canonical computation for $P$ and 
$M$ we obtain that for every ordinal $\alpha$, $X_\alpha\subseteq M$.
Thus, (1), follows.

\noindent
(2) We proceed by induction. The basis is evident as $X_0=\emptyset$.
Let us consider an ordinal $\alpha > 0$ and let us assume that (2) holds 
for every ordinal $\beta <\alpha$. If $\alpha=\beta+1$, then 
$X_\alpha=X_{\beta+1}=\hs(P(X_\beta)) \cap M$. Thus, by the induction 
hypothesis, $X_\beta\subseteq X_{\alpha}$. Since $P$ is a Horn 
constraint program, it follows that 
$P(X_\beta)\subseteq P(X_{\alpha})$. 
Thus
\[
X_\alpha=X_{\beta+1}=\hs(P(X_\beta))\cap M \subseteq \hs(P(X_{\alpha}))
\cap M.
\]
If $\alpha$ is a limit ordinal then for every $\beta<\alpha$, $X_\beta
\subseteq X_\alpha$ and, as before, also $P(X_\beta)\subseteq
P(X_{\alpha})$. Thus, by the induction hypothesis for every $\beta<
\alpha$,
\[
X_\beta \subseteq \hs(P(X_\beta))\cap M \subseteq \hs(P(X_\alpha))\cap M,
\]
which implies that 
\[
X_\alpha = \bigcup_{\beta<\alpha} X_\beta \subseteq \hs(P(X_\alpha))\cap
M.
\]
\end{proof}

Canonical computations have the following {\em fixpoint} property.

\begin{proposition}\label{propnewfp}
Let $P$ be a Horn constraint program. For every model $M$ of $P$, we
have
\[
\hs(P(Can(P,M)))\cap M = Can(P,M).
\]
\end{proposition}
\begin{proof}
Let $\alpha$ be the length of the canonical computation $t^{P,M}$.
Then, $X^{P,M}_{\alpha+1}=X^{P,M}_\alpha=Can(P,M)$. Since 
$X_{\alpha+1}= \hs(X_\alpha)\cap M$, the assertion follows. 
\end{proof}

We now gather properties of derivable models that extend properties of 
the least model of normal Horn logic programs.

\begin{proposition}
\label{grtdm}
Let $P$ be a Horn constraint program. Then:
\begin{enumerate}
\item For every model $M$ of $P$, $Can(P,M)$ is a greatest derivable 
model of $P$ contained in $M$
\item A model $M$ of $P$ is a derivable model if and only if $M=Can(P,M)$
\item If $M$ is a minimal model of $P$ then $M$ is a derivable model of
$P$.
\end{enumerate}
\end{proposition}
\begin{proof}
   (1) Let $M'$ be a derivable model of $P$ such that $M'\subseteq M$.
   Let $T=\langle X_\alpha\rangle$ be a $P$-derivation such that $M'=R_t$.
   We will prove that for every ordinal $\alpha$, $X_\alpha\subseteq
   X^{P,M}_\alpha$. We proceed by transfinite induction. Since $X_0=X^{P,
   M}_0=\emptyset$, the basis for the induction is evident. Let us
   consider an ordinal $\alpha>0$ and assume that for every ordinal
   $\beta <\alpha$, $X_\beta\subseteq X^{P,M}_\beta$.

   If $\alpha=\beta+1$, then $X_{\alpha} \in T^{nd}_P(X_\beta)$ and
   so, $X_{\alpha} \subseteq \hs(P(X_\beta))$. By the induction
   hypothesis and by the monotonicity of the constraints in the bodies of
   rules in $P$, $X_{\alpha}\subseteq \hs(P(X^{P,M}_\beta))$. Thus, since
   $X_\alpha\subseteq R_t=M'\subseteq M$,
   \[
   X_\alpha\subseteq \hs(P(X^{P,M}_\beta))\cap M = X^{P,M}_{\beta+1}=
   X^{P,M}_{\alpha}.
   \]
   The case when $\alpha$ is a limit ordinal is straightforward as
   $X_\alpha=\bigcup_{\beta<\alpha}X_\beta$ and
   $X^{P,M}_\alpha=\bigcup_{\beta<\alpha}X^{P,M}_\beta$.

   \noindent
   (2) ($\Leftarrow$) If $M=Can(P,M)$, then $M$ is the result of the
   canonical $P$-derivation for $P$ and $M$. In particular, $M$ is a
   derivable model of $P$.

   \noindent
   ($\Rightarrow$) if $M$ is a derivable model of $P$, then $M$ is also a
   model of $P$. From (1) it follows that $Can(P,M)$ is the greatest
   derivable model of $P$ contained in $M$. Since $M$ itself is derivable,
   $M=Can(P,M)$.

   \noindent
   (3) From (1) it follows that $Can(P,M)$ is a derivable model of $P$
   and that $Can(P,$ $M) \subseteq M$. Since $M$ is a minimal model,
   $Can(P,M)=M$ and, by $(2)$, $M$ is a derivable model of $P$.
\end{proof}

\subsection{Stable Models}

In this section, we will recall and adapt to our setting the definition
of stable models proposed and studied
by \citeA{mt04} and \citeA{mnt03,mnt06}
Let $P$ be a 
monotone-constraint program and $M$ a subset of $At(P)$. The {\em 
reduct} of $P$, denoted by $P^M$, is a program obtained from $P$ by:
\begin{enumerate}
\item removing from $P$ all rules whose body contains a literal 
$\n(B)$ such that $M\models B$;
\item removing literals $\n(B)$ for the bodies of the remaining 
rules.
\end{enumerate}

The reduct of a monotone-constraint program is Horn since it contains 
no occurrences of default negation. Therefore, the following definition 
is sound.

\begin{definition}
Let $P$ be a monotone-constraint program. A set of atoms $M$ is a {\em 
stable} model of $P$ if $M$ is a derivable model of $P^M$. We denote 
the set of stable models of $P$ by $St(P)$.
\end{definition}

The definitions of the reduct and stable models follow and generalize
those proposed for normal logic programs, since in the setting of Horn 
constraint programs, derivable models play the role of a least model.

As in normal logic programming and its standard extensions, stable
models of monotone-constraint programs are supported models and,
consequently, models.

\begin{proposition}
Let $P$ be a monotone-constraint program. If $M \subseteq At(P)$ is a
stable model of $P$, then $M$ is a supported model of $P$.
\end{proposition}
\begin{proof}
   Let $M$ be a stable model of $P$. Then, $M$ is a derivable model of 
   $P^M$ and, by Proposition \ref{propresmod}, $M$ is a supported model
   of $P^M$.
   It follows that $M$ is a model of $P^M$. Directly from the
   definition of the reduct it follows that $M$ is a model of $P$.
   
   It also follows that $M \subseteq \hs(P^M(M))$. For every rule $r$ 
   in $P^M(M)$, there is a rule $r'$ in $P(M)$, which has the same head 
   and the same non-negated literals in the body as $r$. Thus, $\hs(P^M
   (M))\subseteq \hs(P(M))$ and, consequently, $M\subseteq \hs(P(M))$. It 
   follows that $M$ is a supported model of $P$.
\end{proof}

\noindent
{\bf Examples.}
  Here is an example of stable models of a monotone-constraint program.
  Let $P$ be a monotone-constraint program that contains the following
  rules:
  \begin{quote}
  $2\{a, b, c\}\leftarrow 1\{a, d\}, \n(1\{c\})$\\
  $1\{b, c, d\}\leftarrow 1\{a\}, \n(3\{a, b, d\}))$\\
  $1\{a\}\leftarrow$
  \end{quote}
  Let $M=\{a, b\}$. Therefore, $M\not\models 1\{c\}$ and $M\not\models
  3\{a, b, d\}$. Hence the reduct $P^M$ contains the following three
  Horn rules:
  \begin{quote}
  $2\{a, b, c\}\leftarrow 1\{a, d\}$\\
  $1\{b, c, d\}\leftarrow 1\{a\}$\\
  $1\{a\}\leftarrow$
  \end{quote}
  Since $M=\{a, b\}$ is a derivable model of $P^M$, $M$ is a stable model
  of $P$.

  Let $M'=\{a, b, c\}$. Then $M'\models 1\{c\}$ and $M\not\models
  3\{a,b,d\}$. Therefore, the reduct $P^{M'}$ contains two Horn rules:
  \begin{quote}
  $1\{b, c, d\}\leftarrow 1\{a\}$\\
  $1\{a\}\leftarrow$
  \end{quote}
  Since $M'=\{a,b,c\}$ is a derivable models of $P^{M'}$, $M'$ is also
  a stable model of $P$. We note that stable models of a
  monotone-constraint program, in general, do not form an anti-chain.
  \hfill$\eoe$

If a normal logic program is Horn then its least model is its (only)
stable model. Here we have an analogous situation.

\begin{proposition}
\label{DerIsStable}
Let $P$ be a Horn monotone-constraint program. Then $M \subseteq At(P)$
is a derivable model of $P$ if and only if $M$ is a stable model of $P$.
\end{proposition}
\begin{proof}
   For every set $M$ of atoms $P=P^M$. Thus, $M$ is a derivable model of
   $P$ if and only if it is a derivable model of $P^M$ or, equivalently,
   a stable model of $P$. 
\end{proof}

In the next four sections of the paper we show that several fundamental
results concerning normal logic programs extend to the class of
monotone-constraint programs.

\section{Strong and Uniform Equivalence of Monotone-cons\-traint Programs}

  Strong equivalence and uniform equivalence concern the problem of
  replacing some rules in a logic program with others without changing 
  the overall semantics of the program. More specifically, the strong 
  equivalence concerns replacement of rules within {\em arbitrary} 
  programs, and the uniform equivalence concerns replacements of all 
  {\em non-fact} rules. In each case, the stipulation is that the 
  resulting program must have the same stable models as the original 
  one. Strong (and uniform) equivalence is an important concept 
  due to its potential uses in program rewriting and optimization.

  Strong and uniform equivalence have been studied in the literature
  mostly for normal logic programs \cite{lpv01,lin02,tu03,ef03}.

\citeA{tu03}
presented an elegant characterization of strong equivalence
of {\em smodels} programs, and
\citeA{ef03}
described a similar
characterization of uniform equivalence of normal and disjunctive
logic programs. We show that both characterizations can be adapted to
the case of monotone-constraint programs. In fact, one can show that
under the representations of normal logic programs as monotone-constraint 
programs \cite{mnt03,mnt06} our definitions and characterizations
of strong and uniform equivalence reduce to those introduced and
developed originally for normal logic programs. 

\subsection{{\bfseries {\slshape M}}-maximal Models}

A key role in our approach is played by models of Horn constraint
programs satisfying a certain maximality condition.

\begin{definition}
\label{defmax}
Let $P$ be a Horn constraint program and let $M$ be a model of $P$.
A set $N\subseteq M$ such that $N$ is a model of $P$ and $M\cap
\hs(P(N))\subseteq N$ is an {\em $M$-maximal} model of $P$, written 
$N \models_M P$.
\end{definition}

Intuitively, $N$ is an $M$-maximal model of $P$ if $N$ satisfies each
rule $r\in P(N)$ ``maximally'' with respect to $M$. That is, for every
$r\in P(N)$, $N$ contains all atoms in $M$ that belong to $\hs(r)$ ---
the domain of the head of $r$.

To illustrate this notion, let us consider a Horn constraint program
$P$ consisting of a single rule:
\[
1 \{ p, q, r \} \leftarrow 1 \{ s, t \} .
\]
Let $M= \{ p, q, s, t \}$ and $N=\{ p, q, s \}$. One can verify that 
both $M$ and $N$ are models of $P$. Moreover, since the only rule in 
$P$ is $N$-applicable, and $M \cap \{p, q, r\} \subseteq N$, $N$ is 
an $M$-maximal model of $P$. On the other hand, $N'=\{ p, s \}$
is not $M$-maximal even though $N'$ is a model of $P$ and it is
contained in $M$.


There are several similarities between properties of models of
normal Horn programs and $M$-maximal models of Horn constraint programs.
We state and prove here one of them that turns out to be especially 
relevant to our study of strong and uniform equivalence. 
\begin{proposition}
\label{max}
Let $P$ be a Horn constraint program and let $M$ be a model of $P$. Then
$M$ is an $M$-maximal model of $P$ and $Can(P,M)$ is the least 
$M$-maximal model of $P$.
\end{proposition}
\begin{proof}
The first claim follows directly from the definition. To prove the
second one, we simplify the notation: we will write $N$ for 
$Can(P,M)$ and $X_\alpha$ for $X_\alpha^{P,M}$.

We first show that $N$ is an $M$-maximal model of $P$. Clearly,
$N\subseteq M$. Moreover, by Proposition \ref{propnewfp},
$\hs(P(N))\cap M= N$. Thus, $N$ is indeed an $M$-maximal model of $P$.

We now show $N$ is the least $M$-maximal model of $P$.

Let $N'$ be any $M$-maximal model of $P$. We will show by transfinite 
induction that $N\subseteq N'$. Since $X_0=\emptyset$, the basis for the
induction holds. Let us consider an ordinal $\alpha > 0$ and let us 
assume that $X_\beta\subseteq N'$, for every $\beta<\alpha$. To show
$N\subseteq N'$, it is sufficient to show that $X_{\alpha}\subseteq N'$.

Let us assume that $\alpha=\beta+1$ for some $\beta<\alpha$. Then, since
$X_\beta\subseteq N'$ and $P$ is a Horn constraint program, we have
$P(X_\beta) \subseteq P(N')$. Consequently,
\[
X_\alpha=X_{\beta+1}= \hs(P(X_\beta))\cap M \subseteq \hs(P(N'))\cap M 
\subseteq N',
\]
the last inclusion follows from the fact that$N'$ is an $M$-maximal
model of $P$.

If $\alpha$ is a limit ordinal, then $X_\alpha=\bigcup_{\beta<\alpha}
X_\beta$ and the inclusion $X_\alpha\subseteq N'$ follows directly from
the induction hypothesis.
\end{proof}

\subsection{Strong Equivalence and SE-models}

Monotone-constraint programs $P$ and $Q$ are {\em strongly equivalent},
denoted by $P \equiv_s Q$, if for every monotone-constraint program $R$,
$P\cup R$ and $Q\cup R$ have the same set of stable models.

To study the strong equivalence of monotone-constraint programs, we
generalize the concept of an {\em SE-model}
due to \citeA{tu03}.

  There are close connections between strong equivalence of normal logic 
  programs and the logic here-and-there. The semantics of the 
  logic here-and-there is given in terms of Kripke models with two words
  which, when rephrased in terms of pairs of interpretations (pairs of 
  sets of propositional atoms), give rise to SE-models.
  
\begin{definition}
\label{defse}
Let $P$ be a monotone-constraint program and let $X, Y$ be sets of
atoms. We say that $(X,Y)$ is an {\em SE-model} of $P$ if the following 
conditions hold: (1) $X \subseteq Y$; (2) $Y \models P$; and (3) $X 
\models_Y P^Y$. We denote by $SE(P)$ the set of all SE-models of $P$.
\end{definition}

\noindent
{\bf Examples.}
  To illustrate the notion of an SE-model of a monotone-constraint 
  program, let $P$ consist of the following two rules:
  \begin{quote}
  $2\{p,q,r\}\leftarrow 1\{q, r\}, \n(3\{p,q,r\})\}$\\
  $1\{p,s\}\leftarrow 1\{p, r\}, \n(2\{p,r\})$
  \end{quote}
  We observe that $M=\{p,q\}$ is a model of $P$. Let $N=\emptyset$.
  Then $N\subseteq M$ and $P^M(N)$ is empty. It follows that $M\cap
  \hs(P^M(N)) =\emptyset\subseteq N$ and so, $N\models_M P^M$. Hence,
  $(N,M)$ is an SE-models of $P$.

  Next, let $N'=\{p\}$. It is clear that $N'\subseteq M$. Moreover, 
  $P^M(N')= \{1\{p,s\}\leftarrow 1\{p, r\}\}$. Hence $M\cap \hs(P^M(N'
  ))=\{p\} \subseteq N'$ and so, $N'\models_M P^M$. That is, $(N',M)$ is 
  another SE-model of $P$.
  \hfill$\eoe$
  
SE-models yield a simple characterization of strong equivalence of
monotone-constraint programs. To state and prove it, we need several
auxiliary results.

\begin{lemma}\label{newlemma}
   Let $P$ be a monotone-constraint program and let $M$ be a model of $P$.
   Then $(M,M)$ and $(Can(P^M,M),M)$ are both SE-models of $P$.
   \label{canse}
\end{lemma}
\begin{proof}
   The requirements $(1)$ and $(2)$ of an SE-model hold for $(M,M)$.
   Furthermore, since $M$ is a model of $P$, $M\models P^M$. Finally,
   we also have $\hs(P(M))\cap M\subseteq M$. Thus, $M\models_M P^M$.

   Similarly, the definition of a canonical computation and
   Proposition \ref{propresmod}, imply the first two requirements of the
   definition of SE-models for $(Can(P^M,M), M)$. The third
   requirement follows from Proposition \ref{max}.
\end{proof}

\begin{lemma}
\label{sesteq}
   Let $P$ and $Q$ be two monotone-constraint programs such that $SE(P)=
   SE(Q)$. Then $St(P)=St(Q)$.
\end{lemma}
\begin{proof}
   If $M\in St(P)$, then $M$ is a model of $P$ and, by
   Lemma \ref{newlemma}, $(M,M)\in SE(P)$. Hence, $(M,M)\in SE(Q)$ and,
   in particular, $M\models Q$. By Lemma \ref{newlemma} again,
   \[
   (Can(Q^M, M),M)\in SE(Q).
   \]
   By the assumption,
   \[
   (Can(Q^M,M),M)\in SE(P)
   \]
   and so, $Can(Q^M,M)\models_M P^M$ or, in other terms, $Can(Q^M,M)$ is
   an $M$-maximal model of $P^M$. Since $M\in St(P)$, $M=Can(P^M,M)$.
   By Proposition \ref{max}, $M$ is the least $M$-maximal model of $P^M$.
   Thus, $M \subseteq Can(Q^M,M)$. On the other hand, we have
   $Can(Q^M,M)\subseteq M$ and so, $M=Can(Q^M,M)$. It follows that $M$
   is a stable model of $Q$. The
   other inclusion can be proved in the same way.
\end{proof}

\begin{lemma}
\label{secupcap}
   Let $P$ and $R$ be two monotone-constraint programs. Then $SE(P\cup R)=
   SE(P)\cap SE(R)$.
\end{lemma}
\begin{proof}
   The assertion follows from the following two simple observations.
   First, for every set $Y$ of atoms, $Y\models (P\cup R)$ if and only
   if $Y\models P$ and $Y\models R$. Second, for every two sets $X$ and
   $Y$ of atoms, $X\models_Y(P\cup R)^Y$ if and only if $X\models_Y P^Y$
   and $X\models_Y R^Y$.
\end{proof}

\begin{lemma}
\label{semd}
   Let $P$, $Q$ be two monotone-constraint programs. If $P\equiv_s Q$, then
   $P$ and $Q$ have the same models.
\end{lemma}
\begin{proof}
   Let $M$ be a model of $P$. By $r$ we denote a constraint rule
   $(M,\{M\})\leftarrow\ $. Then, $M\in St(P\cup \{r\})$. Since $P$
   and $Q$ are strongly equivalent, $M\in St(Q\cup \{r\})$. It follows
   that $M$ is a model of $Q\cup \{r\}$ and so, also a model of $Q$.
   The converse inclusion can be proved in the same way.
\end{proof}

%

\begin{theorem}
\label{sethm}
Let $P$ and $Q$ be monotone-constraint programs. Then $P\equiv_s Q$ if
and only if $SE(P)=SE(Q)$.
\end{theorem}
\begin{proof}
($\Leftarrow$) Let $R$ be an arbitrary monotone-constraint program. 
Lemma \ref{secupcap} implies that $SE(P\cup R)= SE(P)\cap SE(R)$ and 
$SE(Q\cup R)=SE(Q)\cap SE(R)$. Since $SE(P)=SE(Q)$, we have that $SE(P
\cup R)= SE(Q\cup R)$. By Lemma \ref{sesteq}, $P\cup R$ and $Q\cup R$ have 
the same stable models.  Hence, $P\equiv_s Q$ holds.

\noindent
($\Rightarrow$)
   Let us assume $SE(P)\setminus SE(Q)\not=\emptyset$ and let us consider
   $(X,Y) \in SE(P)\setminus SE(Q)$. It follows that $X\subseteq Y$ and
   $Y\models P$. By Lemma \ref{semd}, $Y\models Q$. Since $(X,Y)\notin
   SE(Q)$, $X\not\models_Y Q^Y$. It follows that $X\not\models Q^Y$ or
   $\hs(Q^Y(X))\cap Y \not\subseteq X$. In the first case, there is a
   rule $r\in Q^Y(X)$ such that $X\not\models \hd(r)$. Since $X\subseteq
   Y$ and $Q^Y$ is a Horn constraint program, $r\in Q^Y(Y)$. Let us
   recall that $Y\models Q$ and so, we also have $Y\models Q^Y$. It
   follows that $Y\models\hd(r)$. Since
   $\hs(r) \subseteq \hs(Q^Y(X))$, $Y\cap \hs(Q^Y(X))\models \hd(r)$.
   Thus, $\hs(Q^Y(X)) \cap Y \not\subseteq X$ (otherwise, by the
   monotonicity of $\hd(r)$, we would have $X\models \hd(r)$).

   The same property holds in the second case. Thus, it follows that
   \[
   (\hs(Q^Y(X))\cap Y)\setminus X\not=\emptyset.
   \]
   We define
   \[
   X'= (\hs(Q^Y(X))\cap Y)\setminus X.
   \]

   Let $R$ be a constraint program consisting of the following two
   rules:
   \begin{quote}
   $(X,\{X\})\leftarrow$\\
   $(Y,\{Y\})\leftarrow (X',\{X'\})$.
   \end{quote}
   Let us consider a program $Q_0=Q\cup R$. Since $Y\models Q$ and $X
   \subseteq Y$, $Y\models Q_0$. Thus, $Y\models Q_0^Y$ and, in
   particular, $Can(Q_0^Y,Y)$ is well defined. Since $R\subseteq
   Q_0^Y$, $X\subseteq Can(Q_0^Y,Y)$. Thus, we have
   \[
   \hs(Q_0^Y(X))\cap Y \subseteq \hs(Q_0^Y(Can(Q_0^Y,Y))) \cap Y =
   Can(Q_0^Y,Y)
   \]
   (the last equality follows from Proposition \ref{propnewfp}). We
   also have $Q\subseteq Q_0$ and so,
   \[
   X'\subseteq\hs(Q^Y(X))\cap Y \subseteq \hs(Q_0^Y(X))\cap Y.
   \]
   Thus, $X'\subseteq Can(Q_0^Y,Y)$. Consequently, by Proposition
   \ref{propnewfp}, $Y\subseteq Can(Q_0^Y,Y)$. Since
   $Can(Q_0^Y,Y)$ $\subseteq Y$, $Y=Can(Q_0^Y,Y)$ and so, $Y\in St(Q_0)$.

   Since $P$ and $Q$ are strongly equivalent, $Y\in St(P_0)$, where
   $P_0=P\cup R$. Let us recall that $(X,Y)\in SE(P)$. By Proposition
   \ref{max}, $Can(P^Y,Y)$ is a least $Y$-maximal model of $P^Y$. Since
   $X$ is a $Y$-maximal model of $P$ (as $X\models_Y P^Y$), it follows
   that $Can(P^Y,Y)\subseteq X$. Since $X'\not\subseteq X$, $Can(P_0^Y,Y)
   \subseteq X$. Finally, since $X'\subseteq Y$, $Y\not\subseteq X$.  Thus,
   $Y\not=Can(P_0^Y,Y)$, a contradiction.

   It follows that $SE(P)\setminus SE(Q)=\emptyset$. By symmetry,
   $SE(Q)\setminus SE(P)=\emptyset$, too. Thus, $SE(P)=SE(Q)$.
\end{proof}

\subsection{Uniform Equivalence and UE-models}

Let $D$ be a set of atoms. By $r_D$ we denote
a monotone-constraint rule 
\[
r_D = \ \ (D,\{D\})\leftarrow .
\]
Adding a rule $r_D$ to a program forces all atoms in $D$ to be true
(independently of the program).

Monotone-constraint programs $P$ and $Q$ are {\em uniformly equivalent},
denoted by $P \equiv_u Q$, if for every set of {\em atoms} $D$,
$P\cup \{r_D\}$ and $Q\cup \{r_D\}$ have the same stable models.

An SE-model $(X,Y)$ of a monotone-constraint program $P$
is a {\em UE-model} of $P$ if for every SE-model $(X',Y)$ of $P$ with 
$X\subseteq X'$, either $X=X'$ or $X'=Y$ holds. We write $UE(P)$ to 
denote the set of all UE-models of $P$. Our notion of a UE-model is a 
generalization of the notion of a UE-model
due to \citeA{ef03}
to the 
setting of monotone-constraint programs.

\noindent
{\bf Examples.}
  Let us look again at the program we used to illustrate the concept of
  an SE-model. We showed there that $(\emptyset,\{p,q\})$ and $(\{p\},
  \{p,q\})$ are SE-models of $P$. Directly from the definition of
  UE-models it follows that $(\{p\}, \{p,q\})$ is a UE-model of $P$.
  \hfill$\eoe$
  
We will now present a characterization of uniform equivalence of
monotone-con\-straint programs under the assumption that their sets of 
atoms are finite. One can prove a characterization of uniform equivalence
of arbitrary monotone-cons\-traint programs, generalizing one of the
results
by \citeA{ef03}.
However, both the characterization and its proof
are more complex and, for brevity, we restrict our attention to the
finite case only. 

We start with an auxiliary result, which allows us to focus only on 
atoms in $\At(P)$ when deciding whether a pair $(X,Y)$ of sets of atoms
is an SE-model of a monotone-constraint program $P$.

\begin{lemma}
\label{semdat}
Let $P$ be a monotone-constraint program, $X\subseteq Y$ two sets of atoms.
Then $(X, Y)\in SE(P)$ if and only if $(X\cap At(P), Y\cap At(P))\in
SE(P)$.
\end{lemma}
\begin{proof}
   Since $X\subseteq Y$ is given, and $X\subseteq Y$ implies $X\cap
   At(P)\subseteq Y\cap At(P)$, the first condition of the definition
   of an SE-model holds on both sides of the equivalence.

   Next, we note that for every constraint $C$, $Y\models C$ if and 
   only if $Y\cap \dom(C) \models C$. Therefore, $Y\models P$ if and 
   only if $Y\cap At(P) \models P$. That is, the second condition of the
   definition of an SE-model holds for $(X,Y)$ if and only if it holds
   for $(X\cap At(P), Y\cap At(P))$.

   Finally, we observe that $P^Y = P^{Y\cap At(P)}$
   and $P(X) = P(X\cap At(P))$. Therefore,
   \[
   Y\cap \hs(P^Y(X)) = Y\cap \hs(P^{Y\cap At(P)}(X\cap At(P))).
   \]
   Since $\hs(P^{Y\cap At(P)}(X\cap At(P)))\subseteq At(P)$, it 
   follows that
   \[
   Y\cap \hs(P^Y(X))\subseteq X
   \]
   if and only if
   \[
   Y\cap \At(P)\cap \hs(P^{Y\cap \At(P)}(X\cap \At(P))) \subseteq 
   X\cap \At(P).
   \]
   Thus, $X\models_Y P^Y$ if and only if $X\cap At(P)
   \models_{Y \cap At(P)} P^{Y\cap At(P)}$. That is, the third
   condition of the definition of an SE-model holds for $(X,Y)$ if and 
   only if it holds for $(X\cap At(P), Y\cap At(P))$.
\end{proof}

\begin{lemma}
\label{exmax}
Let $P$ be a monotone-constraint program such that $\At(P)$ is finite. 
Then for every $(X,Y)\in SE(P)$ such that $X\not=Y$,
the set
\begin{equation}
\label{eq211}
\{X'\colon X\subseteq X'\subseteq Y,\ X'\not=Y,\ (X',Y)\in SE(P)\}
\end{equation}
has a maximal element.
\end{lemma}
\begin{proof}
If $\At(P)\cap X =\At(P)\cap Y$, then for every element $y\in
Y\setminus X$, $Y\setminus \{y\}$ is a maximal element of the set 
(\ref{eq211}). Indeed,
since $(X, Y)\in SE(P)$, by Lemma \ref{semdat}, $(X\cap At(P), Y\cap
At(P))\in SE(P)$. Since $X\cap At(P)=Y\cap At(P)$ and $y\not\in
At(P)$, $X\cap At(P)=(Y\setminus\{y\})\cap At(P)$. Therefore,
$((Y\setminus\{y\})\cap At(P), Y\cap At(P))\in SE(P)$. Then from Lemma
\ref{semdat} and the fact $Y\setminus\{y\}\subseteq Y$, we have
$(Y\setminus\{y\}, Y)\in SE(P)$.
Therefore, $Y\setminus\{y\}$ belongs to the set (\ref{eq211})
and so, it is a maximal element of this set.

Thus, let us assume that $\At(P)\cap X \not=\At(P)\cap Y$. Let us define
$X'=X\cup (Y\setminus \At(P))$. Then $X\subseteq X' \subseteq Y$ and 
$X'\not=Y$. Moreover, no element in $X'\setminus X$ belongs to $\At(P)$.
That is, $X'\cap At(P)=X\cap At(P)$. Thus, by Lemma \ref{semdat},
$(X',Y)\in SE(P)$ and so, $X'$ belongs to the set (\ref{eq211}).
Since $Y\setminus X'\subseteq \At(P)$, by the finiteness of $\At(P)$ it
follows that the set (\ref{eq211}) contains a maximal element containing
$X'$. In particular, it contains a maximal element.
\end{proof}

\begin{theorem}
\label{uethm}
Let $P$ and $Q$ be two monotone-constraint programs such that
$\At(P)\cup\At(Q)$ is finite. Then $P\equiv_u Q$ if and only if 
$UE(P)=UE(Q)$.
\end{theorem}
\begin{proof}
   ($\Leftarrow$) Let $D$ be an arbitrary set of atoms and $Y$ be a stable
   model of $P\cup \{r_D\}$. Then $Y$ is a model of $P\cup \{r_D\}$.
   In particular, $Y$ is a model of $P$ and so, $(Y,Y) \in UE(P)$. It
   follows that $(Y,Y)\in UE(Q)$, too. Thus, $Y$ is a model of $Q$. Since
   $Y$ is a model of $r_D$, $D\subseteq Y$. Consequently, $Y$ is a model
   of $Q\cup \{r_D\}$ and thus, also of $(Q\cup \{r_D\})^Y$.

   Let $X=Can((Q\cup \{r_D\})^Y,Y)$. Then $D\subseteq X\subseteq Y$ and,
   by Proposition \ref{max}, $X$ is a $Y$-maximal model of $(Q\cup
   \{r_D\})^Y$. Consequently, $X$ is a $Y$-maximal model of $Q^Y$.
   Since $X\subseteq Y$ and $Y\models Q$, $(X,Y)\in SE(Q)$.

   Let us assume
   that $X\not=Y$. Then, by Lemma \ref{exmax}, there is a maximal set
   $X'$ such that $X\subseteq X'\subseteq Y$, $X'\not= Y$ and $(X',Y)\in
   SE(Q)$. It follows that $(X',Y)\in UE(Q)$. Thus, $(X',Y)\in UE(P)$ and
   so, $X'\models_Y P^Y$. Since $D\subseteq X'$, $X'\models_Y (P\cup\{r_D
   \})^Y$. We recall that $Y$ is a stable model of $P\cup \{r_D\}$. Thus,
   $Y=Can((P\cup\{
   r_D\})^Y,Y)$. By Proposition \ref{max}, $Y\subseteq X'$ and so we get
   $X'= Y$, a contradiction. It follows that $X=Y$ and,
   consequently, $Y$ is a stable model of $Q\cup \{r_D\}$.

   By symmetry, every stable model of $Q\cup \{r_D\}$ is also a stable
   model of $P\cup \{r_D\}$.

   \noindent
   ($\Rightarrow$)
   First, we note that $(Y,Y)\in UE(P)$ if and only if $Y$ is a model
   of $P$. Next, we note that $P$ and $Q$ have the same models. Indeed,
   the argument used in the proof of Lemma \ref{semd} works also under
   the assumption that $P\equiv_u Q$. Thus, $(Y,Y)\in UE(P)$ if and only
   if $(Y,Y)\in UE(Q)$.

   Now let us assume that $UE(P)\neq UE(Q)$. Let $(X,Y)$ be an element
   of $(UE(P)\setminus UE(Q))\cup (UE(Q)\setminus UE(P))$.
   Without loss of generality, we can assume that $(X,Y)
   \in UE(P)\setminus UE(Q)$. Since $(X,Y)\in UE(P)$, it follows that
   \begin{enumerate}
   \item $X\subseteq Y$
   \item $Y\models P$ and, consequently, $Y\models Q$
   \item $X\not=Y$ (otherwise, by our earlier observations, $(X,Y)$
   would belong to $UE(Q)$).
   \end{enumerate}

   Let $R=(Q\cup \{r_X\})^Y$. Clearly, $R$ is a Horn constraint program.
   Moreover, since $Y\models Q$ and $X\subseteq Y$, $Y\models R$. Thus,
   $Can(R,Y)$ is defined. We have $X\subseteq Can(R,Y) \subseteq Y$. We
   claim that $Can(R,Y)\neq Y$. Let us assume to the contrary that
   $Can(R,Y)=Y$. Then $Y\in St(Q\cup \{r_X\})$. Hence, $Y\in St(P\cup
   \{r_X\})$, that is, $Y=Can((P\cup \{r_X\})^Y,Y)$. By Proposition
   \ref{max}, $Y$ is the least $Y$-maximal model of $(P\cup \{r_X\})^Y$
   and $X$ is a $Y$-maximal model of $(P\cup \{r_X\})^Y$ (since $(X,Y)\in
   SE(P)$, $X\models_Y P^Y$ and so, $X\models_Y (P\cup \{r_X\})^Y$, too).
   Consequently, $Y \subseteq X$ and, as $X\subseteq Y$, $X=Y$, a
   contradiction.

   Thus, $Can(R,Y)\neq Y$.
   By Proposition \ref{max}, $Can(R,Y)$ is a $Y$-maximal model of $R$.
   Since $Q^Y \subseteq R$, it follows that $Can(R,Y)$ is a $Y$-maximal
   model of $Q^Y$ and so, $(Can(R,Y),Y)\in SE(Q)$. Since $Can(R,Y)\not=
   Y$, from Lemma \ref{exmax} it follows that there is a maximal set $X'$
   such that $Can(R,Y)\subseteq X'\subseteq Y$, $X'\not=Y$ and $(X',Y)\in
   SE(Q)$. By the definition, $(X',Y)\in UE(Q)$. Since $(X,Y)\notin
   UE(Q)$. $X\not=X'$. Consequently, since $X\subseteq X'$, $X'\not=Y$
   and $(X,Y)\in UE(P)$, $(X',Y)\notin UE(P)$.

   Thus, $(X',Y)\in UE(Q)\setminus UE(P)$. By applying now the same
   argument as above to $(X',Y)$ we show the existence of $X''$ such
   that $X'\subseteq X''\subseteq Y$, $X'\not=X''$,
   $X''\not=Y$ and $(X'',Y)\in SE(P)$. Consequently, we have $X\subseteq
   X''$, $X\not= X''$ and $Y\not=X''$, which contradicts the fact that 
   $(X,Y)\in UE(P)$. It follows then that $UE(P)=UE(Q)$.
\end{proof}

\noindent
{\bf Examples.}
Let $P=\{ 1 \{ p, q \} \leftarrow \n(2 \{ p, q \}) \}$, and
      $Q = \{ p \leftarrow \n(q)$, $q \leftarrow \n(p) \}$.
Then $P$ and $Q$ are strongly equivalent. We note that both programs 
have
$\{ p \}$, $\{ q \}$, and $\{ p, q \}$ as models. Furthermore, $(\{p\},
\{p\})$, $(\{q\},\{q\})$, $(\{p\},\{p,q\})$, $(\{q\},\{p,q\})$,
$(\{p,q\},\{p,q\})$ and $(\emptyset,\{p,q\})$ are ``all'' SE-models of 
the two programs
\footnote{From Lemma \ref{semdat} and Theorem \ref{sethm}, it follows
that only those SE-models that contain atoms only from $At(P)\cup At(Q)$
are the essential ones.}.

Thus, by Theorem
\ref{sethm}, $P$ and $Q$ are strongly equivalent.

We also observe that the first five SE-models are precisely UE-models
of $P$ and $Q$. Therefore, by Theorem \ref{uethm}, $P$ and $Q$ are also 
uniformly equivalent.

It is possible for two monotone-constraint programs to be uniformly 
but not strongly equivalent. If we add rule $p \leftarrow $ 
to $P$, and rule $p \leftarrow q$ to $Q$, then the two resulting
programs, say $P'$ and $Q'$, are uniformly equivalent. However, they 
are not strongly 
equivalent. The programs $P'\cup\{ q \leftarrow p \}$ and $Q'\cup\{ q
\leftarrow p \}$ have different stable models. Another way to show it 
is by observing that $(\emptyset, \{p, q\})$ is an SE-model of $Q'$ but 
not an SE-model of $P'$.
\hfill$\eoe$

\section{Fages Lemma}

In general, supported models and stable models of a logic program (both
in the normal case and the monotone-constraint case) do not coincide. 
Fages Lemma \cite{fag94}, later extended by \citeA{el03},
establishes a sufficient condition under which a supported model of a
normal logic program is stable. In this section, we show that Fages 
Lemma extends to programs with monotone constraints.

\begin{definition}
A monotone-constraint program $P$ is called {\em tight}
on a set $M \subseteq
At(P)$ of atoms, if there exists a mapping $\lambda$ from $M$ to
ordinals such that for every rule $A \leftarrow A_1, \ldots, A_k, 
\n(A_{k+1}),$ $\ldots,\n(A_m)$ in $P(M)$, if $X$ is the domain of $A$ and 
$X_i$ the domain of $A_i$, $1\leq  i\leq k$, then for every $x \in M \cap
X$ and for every $a \in M \cap \bigcup_{i=1}^{k} X_i$, $\lambda(a) < 
\lambda(x)$.
\end{definition}

We will now show that tightness provides a sufficient condition for
a supported model to be stable. In order to prove a general result, we
first establish it in the Horn case.

\begin{lemma} \label{fages.horn}
Let $P$ be a Horn monotone-constraint program and let $M$ be a supported
model of $P$. If $P$ is tight on $M$, then $M$ is a stable model of $P$.
\end{lemma}
\begin{proof}
   Let $M$ be an arbitrary supported model of $P$ such that $P$ is tight
   on $M$. Let $\lambda$ be a mapping showing the tightness of $P$ on $M$.
   We will show that for every ordinal $\alpha$ and for every atom $x\in M$
   such that $\lambda(x)\leq \alpha$, $x\in Can(P,M)$. We will proceed by
   induction.

   For the basis of the induction, let us consider an atom $x\in M$ such
   that $\lambda(x)=0$. Since $M$ is a supported model for $P$ and
   $x \in M$, there exists a rule $r\in P(M)$ such that $x\in \hs(r)$.
   Moreover, since $P$ is tight on $M$, for every $A\in \bd(r)$ and for
   every $y\in \dom(A)\cap M$, $\lambda(y) < \lambda(x) = 0$. Thus, for
   every $A\in \bd(r)$, $\dom(A)\cap M=\emptyset$. Since $M\models \bd(r)$
   and since $P$ is a Horn monotone-constraint program, it follows that
   $\emptyset\models \bd(r)$. Consequently, $\hs(r)\cap M \subseteq Can
   (P,M)$ and so, $x\in Can(P,M)$.

   Let us assume that the assertion holds for every ordinal $\beta <
   \alpha$ and let us consider $x\in M$ such that $\lambda(x)=\alpha$.
   As before, since $M$ is a supported model of $P$, there exists a rule
   $r\in P(M)$ such that $x\in \hs(r)$. By the assumption, $P$ is tight
   on $M$ and, consequently, for every $A\in \bd(r)$ and for every $y\in
   \dom(A)\cap M$, $\lambda(y) < \lambda(x)=\alpha$. By the induction
   hypothesis, for every $A\in \bd(r)$, $\dom(A)\cap M \subseteq Can(P,M)$.
   Since $P$ is a Horn monotone-constraint program, $Can(P,M)\models
   \bd(r)$. By Proposition \ref{propnewfp}, $\hs(r)\cap M\subseteq Can(P,
   M)$ and so, $x\in Can(P,M)$.

   It follows that $M\subseteq Can(P,M)$. By the definition of a
   canonical computation, we have $Can(P,M) \subseteq M$. Thus, $M=Can(P,
   M)$. By Proposition \ref{DerIsStable}, $M$ is a stable model of $P$.
\end{proof}

Given this lemma, the general result follows easily.

\begin{theorem}
\label{fages.thm}
Let $P$ be a monotone-constraint program and let $M$ be a supported
model of $P$. If $P$ is tight on $M$, then $M$ is a stable model of $P$.
\end{theorem}
\begin{proof}
One can check that if $M$ is a supported model of $P$, then it
is a supported model of the reduct $P^M$. Since $P$ is tight on $M$,
the reduct $P^M$ is tight on $M$, too. Thus, $M$ is a stable model of
$P^M$ (by Lemma \ref{fages.horn}) and, consequently, a derivable model 
of $P^M$ (by Proposition \ref{DerIsStable}). It follows that $M$ is a 
stable model of $P$.
\end{proof}

\section{Logic $\plmc$ and the Completion of 
a Monotone-con\-straint Program}
\label{secplmc}

The {\em completion} of a normal logic program \cite{cl78} is a 
propositional theory whose models are precisely supported models of the 
program. Thus, supported models of normal logic programs can be computed 
by means of SAT solvers. Under some conditions, for instance, when the 
assumptions of Fages Lemma hold, supported models are stable. Thus,
computing models of the completion can yield stable models, an idea 
implemented in the first version of {\em cmodels} software
\cite{cmodels}.

Our goal is to extend the concept of the completion to programs with 
monotone constraints. The completion, as we define it, retains much of
the structure of monotone-constraint rules and allow us, in the 
restricted setting of {\em lparse} programs, to use pseudo-boolean 
constraint solvers to compute supported models of such programs. In this 
section we define the completion and prove a result relating supported 
models of programs to models of the completion. We discuss extensions
of this result in the next section and their practical computational 
applications in Section \ref{sec-appl}.

To define the completion, we first introduce an extension of
propositional logic with monotone constraints, a formalism we denote by 
$\plmc$. A {\em formula} in the logic $\plmc$ is an expression built 
from monotone constraints by means of boolean connectives $\wedge$,
$\vee$ (and their {\em infinitary} counterparts), $\rightarrow$ and 
$\neg$. The notion of a model of a constraint, which we discussed 
earlier, extends in a standard way to the class of formulas in the 
logic $\plmc$.

For a set $L =\{A_1,\ldots,A_k, \n(A_{k+1}),\ldots, \n(A_m)\}$ of 
literals, we define
\[
L^\wedge = A_1\wedge \ldots\wedge A_k\wedge \neg
A_{k+1}\wedge\ldots\wedge \neg A_m.
\]
Let $P$ be a monotone-constraint program. We form the {\em completion} 
of $P$, denoted $\comp(P)$, as follows:
\begin{enumerate}
\item For every rule $r\in P$ we include in $\comp(P)$ a $\plmc$ formula
\[
[\bd(r)]^\wedge \rightarrow \hd(r)
\]
\item For every atom $x\in \At(P)$, we include in $\comp(P)$ a $\plmc$
formula 
\[
x \rightarrow \bigvee \{[\bd(r)]^\wedge\colon r\in P, x\in 
\hs(r)\}
\]
(we note that when the set of rules in $P$ is infinite, the 
disjunction may be infinitary).
\end{enumerate}

The following theorem generalizes a fundamental result on the program
completion from normal logic programming \cite{cl78} to the case of
programs with monotone constraints.
\begin{theorem}
\label{cmp.thm}
     Let $P$ be a monotone-constraint program. A set $M\subseteq \At(P)$
     is a supported model of $P$ if and only if $M$ is a model
     of $\comp(P)$.
\end{theorem}
\begin{proof}
   $(\Rightarrow)$ Let us suppose that $M$ is a supported model of $P$.
   Then $M$ is a model of $P$, that is, for each rule $r \in P$, if $M
   \models \bd(r)$ then $M \models \hd(r)$. Since $M\models \bd(r)$ if
   and only if $M\models[\bd(r)]^\wedge$, it follows that all formulas
   in $\comp(P)$ of the first type are satisfied by $M$.

   Moreover, since $M$ is a supported model of $P$, $M \subseteq
   \hs(P(M))$. That is, for every atom $x\in M$, there exists at least
   one rule $r$ in $P$ such that $x\in\hs(r)$ and $M\models\bd(r)$.
   Therefore, all formulas in $\comp(P)$ of the second type are satisfied
   by $M$, too.

   \noindent
   $(\Leftarrow)$ Let us now suppose that $M$ is a model of $\comp(P)$.
   Since $M\models \bd(r)$ if and only if $M\models[\bd(r)]^\wedge$,
   and since $M$ satisfies formulas of the first type in $\comp(P)$,
   $M$ is a model of $P$.

   Let $x\in M$. Since $M$ satisfies the formula $x \rightarrow \bigvee
   \{[\bd(r)]^\wedge\colon r\in P, x\in \hs(r)\}$, it follows that
   $M$ satisfies $\bigvee \{[\bd(r)]^\wedge\colon r\in P, x\in\hs(r)\}$.
   That is, there is $r\in P$ such that $M$ satisfies $[\bd(r)]^\wedge$
   (and so, $\bd(r)$, too) and $x\in\hs(r)$. Thus, $x\in \hs(P(M))$.
   Hence, $M$ is a supported model of $P$. 
\end{proof}

  Theorems \ref{fages.thm} and \ref{cmp.thm} have the following
  corollary.

  \begin{corollary}
  Let $P$ be a monotone-constraint program. A set $M\subseteq At(P)$ is a
  stable model of $P$ if $P$ is tight on $M$ and $M$ is a model of $\comp(P)$. 
  \end{corollary}
  
We observe that for the material in this section it is not necessary to 
require that constraints appearing in the bodies of program rules be 
monotone. However, since we are only interested in this case, we adopted 
the monotonicity assumption here, as well.

\section{Loops and Loop Formulas in Monotone-constraint Programs}
\label{secloop}

The completion alone is not quite satisfactory as it relates {\em 
supported} not {\em stable} models of monotone-constraint programs with 
models of $\plmc$ theories. Loop formulas, proposed
by \citeA{lz02},
provide a way to eliminate those supported models of normal logic 
programs, which are not stable. Thus, they allow us to use SAT solvers 
to compute stable models of {\em arbitrary} normal logic programs and 
not only those, for which supported and stable models coincide.

We will now extend this idea to monotone-constraint programs. In this
section, we will restrict our considerations to programs $P$ that are 
{\em finitary}, that is, $\At(P)$ is finite. This restriction implies
that monotone constraints that appear in finitary programs have finite 
domains.

Let $P$ be a finitary monotone-constraint program. The {\em positive 
dependency graph} of $P$ is the directed graph $G_P=(V,E)$, where $V=
At(P)$ and $\langle u, v \rangle$ is an edge in $E$ if there exists a 
rule $r\in P$ such that $u\in \hs(r)$ and $v\in \dom(A)$ for some
monotone constraint $A\in \bd(r)$ (that is, $A$ appears non-negated in
  $\bd(r)$). We note that positive dependency graphs of finitary
programs are finite.

Let $G=(V,E)$ be a directed graph. A set $L\subseteq V$ is a {\em loop} 
in $G$ if the subgraph of $G$ induced by $L$ is strongly connected.
A loop is {\em maximal} if it is not a 
proper subset of any other loop in $G$. Thus, maximal loops are vertex 
sets of strongly connected components of $G$. 
A maximal loop is {\em terminating} if there is no edge in $G$ 
from $L$ to any other maximal loop.

These concepts can be extended to the case of programs. By a {\em loop} 
({\em maximal loop}, {\em terminating loop}) of a monotone-constraint 
program $P$, we mean the loop (maximal loop, terminating loop) of the 
positive dependency graph $G_P$ of $P$. We observe that every finitary
monotone-constraint program $P$ has a terminating loop, since $G_P$ is
finite.

Let $X\subseteq \At(P)$. By $G_P[X]$  we denote the subgraph of $G_{P}$ 
{\em induced} by $X$. We observe that if $X\not=\emptyset$ then every
loop of $G_P[X]$ is a loop of $G_P$.

Let $P$ be a monotone-constraint program $P$. For every model $M$ of
$P$ (in particular, for every model $M$ of $\comp(P)$), we define $M^-=
M \setminus Can(P^M,M)$. Since $M$ is a model of $P$, $M$ is a model of 
$P^M$. Thus, $Can(P^M,M)$ is well defined and so is $M^-$.

For every loop in the graph $G_P$ we will now define the corresponding
loop formula. First, for a constraint $A=(X,C)$ and a set $L\subseteq
\At$, we set $A_{|L}=(X, \{Y\in C\colon Y\cap L=\emptyset\})$ and
call $A_{|L}$ the {\em restriction} of $A$ to $L$. Next, let $r$ be a
monotone-constraint rule, say
\[
r=\ \ A \leftarrow A_1,\ldots,A_k,\n(A_{k+1}),\ldots, \n(A_m).
\]
If $L\subseteq \At$, then define a $\plmc$ formula $\bdf_L(r)$
by setting
\[
\bdf_L(r) = {A_1}_{|L}\wedge\ldots\wedge {A_k}_{|L}\wedge \neg A_{k+1}
\wedge \ldots \wedge \neg A_m.
\]

Let $L$ be a loop of a monotone-constraint program $P$. Then, the {\em 
loop formula} for $L$, denoted by $LP(L)$, is the $\plmc$ formula
\[
LP(L) = \bigvee L \rightarrow \bigvee \{\bdf_L(r)\colon r\in P\ \mbox{
and}\ L\cap \hs(r) \neq \emptyset\} 
\]
(we recall that we use the convention to write $a$ for the constraint
$C(a) = (\{a\},$ $\{\{a\}\})$. A {\em loop completion} of a finitary 
monotone-constraint program $P$ is the $\plmc$ theory
\[
LComp(P) = \comp(P)\cup \{LP(L) \colon \mbox{$L$ is a loop in $G_P$}\}.
\]

The following theorem exploits the concept of a loop formula to provide 
a necessary and sufficient condition for a model being a stable model. 
transfinite one.
\begin{theorem}
\label{loop.thm}
     Let $P$ be a finitary monotone-constraint program. A set $M\subseteq
     \At(P)$ is a stable model of $P$ if and only if $M$ is a
     model of $LComp(P)$.
\end{theorem}
\begin{proof}
   $(\Rightarrow)$ Let $M$ be a stable model of $P$. Then $M$ is a
   supported model of $P$ and, by Theorem \ref{cmp.thm}, $M \models
   \comp(P)$.

   Let $L$ be a loop in $P$.
   If $M\cap L=\emptyset$ then $M\models LP(L)$.
   Thus, let us assume that $M\cap L\not=\emptyset$.
   Since $M$ is a stable model of $P$, $M$ is a derivable model of
   $P^M$, that is, $M=Can(P^M,M)$. Let $(X_n)_{n=0,1,\ldots}$ be the
   canonical $P^M$-derivation with respect to $M$ (since we assume that
   $P$ is finite and each constraint in $P$ has a finite domain,
   $P$-derivations reach their results in finitely many steps). Since
   $Can(P^M,M) \cap L = M\cap L \not=\emptyset$, there is a smallest
   index $n$ such that $X_n\cap L\not=\emptyset$. In particular, it
   follows that $n>0$ (as $X_0=\emptyset$) and $L\cap X_{n-1}=
   \emptyset$.

   Since $X_n=\hs(P^M(X_{n-1})\cap M$ and $X_n\cap L\not=\emptyset$, there
   is a rule $r\in P^M(X_{n-1})$ such that $\hs(r)\cap L\not=\emptyset$,
   that is, such that $L\cap \hs(r)) \neq \emptyset$.
   Let $r'$ be a rule in $P$, which
   contributes $r$ to $P^M$. Then, for every literal $\n(A)\in \bd(r')$,
   $M\models \n(A)$. Let $A\in \bd(r')$. Then $A\in \bd(r)$ and so,
   $X_{n-1}\models A$. Since $X_{n-1}\cap L=\emptyset$, $X_{n-1}\models
   A_{|L}$, too, By the monotonicity of $A_{|L}$, $M\models A_{|L}$.
   Thus, $M\models\bdf_L(r')$. Since $\hs(r')\cap L\not=\emptyset$,
   $L\cap \hs(r)) \neq \emptyset$
   and so, $M\models LP(L)$. Thus, $M\models LComp(P)$.

   \noindent
   $(\Leftarrow)$ Let us consider a set $M\subseteq \At(P)$ such that
   $M$ is not a stable model of $P$. If $M$ is not a supported model of
   $P$ that $M\not\models \comp(P)$ and so $M$ is not a model of
   $LComp(P)$. Thus, let us assume that $M$ is a supported
   model of $P$. It follows that $M^-\not=\emptyset$. Let $L\subseteq
   M^-$ be a terminating loop for $G_P[M^-]$.

   Let $r'$ be an arbitrary rule in $P$ such that $L\cap \hs(r')) \neq
   \emptyset$, and let $r$ be the rule obtained from $r'$ by removing
   negated constraints from its body. Now, let us assume that $M\models
   \bdf_{r'}(L)$. It follows that for every literal $\n(A)\in \bd(r')$, $M
   \models \n(A)$. Thus, $r\in P^M$. Moreover, since $L$ is a terminating
   loop for $G_P[M^-]$, for every constraint $A\in \bd(r')$, $\dom(A)\cap
   M^- \subseteq L$. Since $M\models A_{|L}$, it follows that $Can(P^M,M)
   \models A$. Consequently, $\hs(r')\cap L \subseteq
   \hs(r')\cap M\subseteq Can(P^M,M)$ and so, $L\cap Can(P^M,M)\not=
   \emptyset$, a contradiction. Thus, $M\not\models \bigvee\{\bdf_{r'}(L)
   \colon r'\in P\ \mbox{and}\ L\cap \hs(r')) \neq \emptyset\}$. Since
   $M\models \bigvee L$, it follows that $M\not\models LP(L)$ and so,
   $M\not\models LComp(P)$.
\end{proof}

The following result follows directly from the proof of Theorem
\ref{loop.thm} and provides us with a way to filter out specific 
non-stable supported models from $\comp(P)$.

\begin{theorem}
\label{loop.cor}
     Let $P$ be a finitary monotone-constraint program and $M$ a model
     of $\comp(P)$. If $M^-$ is not empty, then $M$ violates the loop
     formula of every terminating loop of $G_P[M^-]$.
\end{theorem}

Finally, we point out that, Theorem \ref{loop.thm} does not hold
when a program $P$ contains infinitely many rules. Here is a counterexample:

\noindent
{\bf Examples.}
Let $P$ be the set of following rules:
\begin{quote}
\noindent
$1\{a_0\} \leftarrow 1\{a_1\}$\\
$1\{a_1\} \leftarrow 1\{a_2\}$\\
$\cdots$\\
$1\{a_n\} \leftarrow 1\{a_{n+1}\}$\\
$\cdots$
\end{quote}

Let $M=\{a_0,\ldots,a_n,\ldots\}$. Then $M$ is a supported model of $P$.
The only stable model of $P$ is $\emptyset$. However, $M^-=M\setminus
\emptyset$ does not contain any terminating loop. The problem arises
because there is an infinite simple path in $G_P[{M^-}]$. Therefore,
$G_P[{M^-}]$ does not have a sink, yet it does not have a terminating
loop either.
\hfill$\eoe$

  The results of this section, concerning the program completion and 
  loop formulas --- most importantly, the loop-completion theorem ---
  form the basis of a new software system to compute stable models of
  {\em lparse} programs. We discuss this matter in Section \ref{sec-appl}.

\section{Programs with Convex Constraints}
\label{secconvex}

We will now discuss programs with convex constraints, which are closely
related to programs with monotone constraints. Programs with convex
constraints are of interest as they do not involve explicit occurrences 
of the default negation operator $\n$, yet are as expressive as programs 
with monotone-constraints. Moreover, they directly subsume an essential
fragment of the class of {\em lparse} programs \cite{sns02}.

A constraint $(X,C)$ is {\em convex}, if for every $W,Y,Z \subseteq X$ 
such that $W \subseteq Y \subseteq Z$ and $W,Z\in C$, we have $Y \in
C$.
A constraint rule of the form
(\ref{eq1a})
is a 
{\em convex-constraint rule} if $A$, $A_1,\ldots,A_n$ are convex
constraints and $m=k$.
Similarly, a constraint program built of
convex-constraint rules is a {\em convex-constraint program}.

The concept of a model discussed in Section \ref{prel} applies to
convex-constraint programs. To define supported and stable models 
of convex-constraint programs, we view them as special programs with 
monotone-constraints.

To this end, we define the {\em upward} and {\em downward closures} of 
a constraint $A=(X,C)$ to be constraints $A^+=(X,C^+)$ and $A^-=(X,C^-)$,
respectively, where
\begin{quote}
$C^+= \{Y\subseteq X\colon \mbox{for some $W\in C$, $W\subseteq Y$}\}$,
{and}\\
$C^-= \{Y\subseteq X\colon \mbox{for some $W\in C$, $Y\subseteq W$}\}$.
\end{quote}
We note that the constraint $A^+$ is monotone. We call a constraint
$(X,C)$ {\em antimonotone} if $C$ is closed under subset, that is, for 
every $W, Y \subseteq X$, if $Y \in C$ and $W \subseteq Y$ then $W\in 
C$. It is clear that the constraint $A^-$ is antimonotone. 

The upward and downward closures allow us to represent any convex 
constraint as the ``conjunction'' of a monotone constraint and an 
antimonotone constraint.Namely, we have the following property of convex 
constraints.

\begin{proposition}
\label{can}
A constraint $(X,C)$ is convex if and only if $C=C^+ \cap C^-$.
\end{proposition}
\begin{proof}
($\Leftarrow$) Let us assume that $C=C^+\cap C^-$and let us consider
a set $M$ such that $M'\subseteq M \subseteq M''$, where $M',M''\in C$. 
it follows that $M'\in C^+$ and $M''\in C^-$. Thus, $M\in C^+$ and $M
\in C^-$. Consequently, $M\in C$, which implies that $(X,C)$ is 
convex.

\noindent
($\Rightarrow$) The definitions directly imply that $C\subseteq C^+$ 
and $C\subseteq C^-$. Thus, $C\subseteq C^+\cap C^-$. Let us consider
$M\in C^+\cap C^-$. Then there are sets $M',M''\in C$ such that $M'
\subseteq M$ and $M\subseteq M''$. Since $C$ is convex, $M\in C$.
Thus, $C^+\cap C^-\subseteq C$ and so, $C=C^+\cap C^-$.
\end{proof}

Proposition \ref{can} suggests an encoding of convex-constraint programs as 
monotone-constraint programs. To present it, we need more notation. 
For a constraint $A=(X,C)$, we call the constraint $(X,\overline{C})$, 
where $\overline{C}= {\cal P}(X) \setminus C$, the {\em dual constraint} 
for $A$. We denote it by $\overline{A}$. 
It is a direct consequence of the
definitions that a constraint $A$ is monotone if and only if its
dual $\overline{A}$ is antimonotone.

Let $C$ be a convex constraint. We set $mc(C)=\{C\}$ if $C$ is monotone.
We set $mc(C)=\{\n(\overline{C})\}$, if $C$ is antimonotone. We define 
$mc(C)=\{C^+,\n(\overline{C^-})\}$, if $C$ is neither monotone nor 
antimonotone. Clearly, $C$ and $mc(C)$ have the same models.

Let $P$ be a convex-constraint program. By $mc(P)$ we denote the program 
with monotone constraints obtained by replacing every rule 
$r$ in $P$ with a rule $r'$ such that
\[
\hd(r')=\hd(r)^+\ \ \mbox{and}\ \ \bd(r')=\bigcup\{mc(A)\colon A\in
\bd(r)\}
\]
and, if $\hd(r)$ is {\em not} monotone, also with an additional rule
$r''$ such that
\[
\hd(r'')= (\emptyset, \emptyset)\ \ \mbox{and}\ \ \bd(r'')=
\{\overline{\hd(r)^-}\}\cup \bd(r'). 
\]
By our observation above, all constraints appearing in rules of $mc(P)$ 
are indeed monotone, that is, $mc(P)$ is a program with monotone 
constraints.

It follows from Proposition \ref{can} that $M$ is a model of $P$ if and 
only if $M$ is a model of $mc(P)$. We extend this correspondence to 
supported and stable models of a convex constraint program $P$ and the
monotone-constraint program $mc(P)$.

\begin{definition}
  Let $P$ be a convex constraint program. Then a set of atoms $M$ is
  a supported (or stable) model of $P$ if $M$ is a supported (or
  stable) model of $mc(P)$.
\end{definition}

With these definitions, monotone-constraint programs can be viewed 
(almost) directly as convex-constraint programs. Namely, we note that 
monotone and antimonotone constraints are convex. Next, we observe that 
if $A$ is a monotone constraint, the expression $\n(A)$ has the same 
meaning as the antimonotone constraint $\overline{A}$ in the sense that 
for every interpretation $M$, $M\models \n(A)$ if and only if $M\models 
\overline{A}$. 

Let $P$ be a monotone-constraint program. By $cc(P)$ we denote the 
program obtained from $P$ by
replacing every rule $r$ of the form (\ref{eq1a}) in $P$ with $r'$
such that
\[
\hd(r')=\hd(r)\ \ \mbox{and}\ \ \bd(r')=\bigcup\{A_i\colon i=1,\ldots,k\}
\cup \bigcup\{\overline{A_j}\colon j=k+1,\ldots,m\}
\]
One can show that 
programs $P$ and $cc(P)$ have the same models, supported models and 
stable models. In fact, for every monotone-constraint program $P$ we 
have $P=mc(cc(P))$. 

\noindent
{\bf Remark.}
Another consequence of our discussion is that the default negation 
operator can be eliminated from the syntax at the price of allowing 
antimonotone constraints and using antimonotone constraints as negated 
literals.
\hfill$\Box$

Due to the correspondences we have established above, one can extend to 
convex-constraint programs all concepts and results we discussed earlier 
in the context of monotone-constraint programs. In many cases, they can 
also be stated {\em directly} in the language of convex-constraints. The 
most important for us are the notions of the completion and loop formulas, 
as they lead to new algorithms for computing stable models of {\em 
lparse} programs. Therefore, we will now discuss them in some detail.

As we just mentioned, we could use $\comp(mc(P))$ as a definition of 
the completion $\comp(P)$ for a convex-constraint logic program $P$. 
Under this definition Theorems \ref{convexloop.thm} extends to the case 
of convex-constraint programs. However, $\comp(mc(P))$ involves monotone 
constraints and their negations and {\em not} convex constraints that
appear in $P$. Therefore, we will now propose another approach, which 
preserves convex constraints of $P$.

To this end, we first extend the logic $\plmc$ with convex constraints. 
In this extension, which we denote by $\plcc$ and refer to as the {\em 
propositional logic with convex-constraints}, formulas are boolean 
combinations of convex constraints. The semantics of such formulas is 
given by the notion of a model obtained by extending over boolean 
connectives the concept of a model of a convex constraint.

Thus, the only difference between the logic $\plmc$, which we used to 
define the completion and loop completion for monotone-convex programs 
and the logic $\plcc$ is that the former uses monotone constraints as 
building blocks of formulas, whereas the latter is based on convex 
constraints. In fact, since monotone constraints are special convex 
constraints, the logic $\plmc$ is a fragment of the logic $\plcc$.

Let $P$ be a convex-constraint program. The completion of $P$,
denoted by \\
$\comp(P)$, is the following set of $\plcc$ formulas:

\begin{enumerate}
\item For every rule $r\in P$ we include in $\comp(P)$ a $\plcc$ formula
\[
[\bd(r)]^\wedge \rightarrow \hd(r)
\]
(as before, for a set of convex constraints $L$, $L^\wedge$ denotes the
conjunction of the constraints in $L$)
\item For every atom $x\in \At(P)$, we include in $\comp(P)$ a $\plcc$
formula 
\[
x \rightarrow \bigvee \{[\bd(r)]^\wedge\colon r\in P,\ x\in 
\hs(r)\}
\]
(again, we note that when the set of rules in $P$ is infinite, the 
disjunction may be infinitary).
\end{enumerate}

One can now show the following theorem.

\begin{theorem}
Let $P$ be a convex-constraint program and let $M\subseteq \At(P)$.
Then $M$ is a supported model of $P$ if and only if 
$M$ is a model of $\comp(P)$.
\end{theorem}
\begin{proof}
(Sketch) By the definition, $M$ is a supported model of $P$ if and only 
if $M$ is a supported model of $mc(P)$. It is a matter of routine 
checking that $\comp(mc(P))$ and $\comp(P)$ have the same models. Thus
the assertion follows from Theorem \ref{cmp.thm}.
\end{proof}

Next, we restrict attention to {\em finitary} convex-constraint programs,
that is, programs with finite set of atoms, and extend to this class of 
programs the notions of the positive dependency graph and loops. 
%
To this end,
we exploit 
its representation as a monotone-constraint program $mc(P)$. That is,
we define the positive dependency graph, loops and loop formulas for $P$
as the positive dependency graph, loops and loop formulas of $mc(P)$, 
respectively. In particular, $L$ is a loop of $P$ if and only if $L$ is 
a loop of $mc(P)$ and the loop formula for $L$, with respect to a 
convex-constraint program $P$, is defined as the loop formula $LP(L)$ 
with respect to the program $mc(P)$\footnote{There is one minor
simplification one might employ. For a monotone constraint $A$, $\neg A$
and $\overline{A}$ are equivalent and $\overline{A}$ is antimonotone and
so, convex. Thus, we can eliminate the operator $\neg$ from loop
formulas of convex-constraint programs by writing $\overline{A}$ instead
of $\neg A$.}. We note that since loop formulas for monotone-constraint 
programs only modify non-negated literals in the bodies of rules and
leave negated literals intact, there seems to be no simple way to extend 
the notion of a loop formula to the case of a convex-constraint program 
$P$ without making references to $mc(P)$.

We now define a {\em loop completion} of a finitary convex-constraint 
program $P$ as the $\plcc$ theory
\[
LComp(P) = \comp(P)\cup \{LP(L) \colon \mbox{$L$ is a loop of $P$}\}.
\]

We have the following theorem that provides a necessary and 
sufficient condition for a set of atoms to be a stable model of a
convex-constraint program.

\begin{theorem}
\label{convexloop.thm}
      Let $P$ be a finitary convex-constraint program. A set $M\subseteq
      \At(P)$ is a stable model of $P$ if and only if $M$ is a model of
      $LComp(P)$.
\end{theorem}
\begin{proof}
    (Sketch) 
    Since $M$ is a stable model of $P$ if and only of $M$ is a stable
    model of $mc(P)$, Theorem \ref{loop.thm} implies that $M$ is a 
    stable model of $P$ if and only if $M$ is a stable model of 
    $LComp(mc(P))$. It is a matter of routine checking that $LComp(mc
    (P))$ and $LComp(P)$ have the same models. Thus, the result follows.
\end{proof}

In a similar way, Theorem \ref{loop.cor} implies the following result for
convex-constraint programs.

\begin{theorem}
\label{convexloop.cor}
      Let $P$ be a finitary convex-constraint program and $M$ a model
      of $\comp(P)$. If $M^-$ is not empty, then $M$ violates the loop
      formula of every terminating loop of $G_P[M^-]$.
\end{theorem}

We emphasize that one could simply use $LComp(mc(P))$ as a definition 
of the loop completion for a convex-constraint logic program. However,
our definition of the completion component of the loop completion
retains the structure of constraints in a program $P$, which is
important when using loop completion for computation of stable
models, the topic we address in the next section of the paper.

\section{Applications}
\label{sec-appl}

In this section, we will use theoretical results on the program
completion, loop formulas and loop completion of programs with convex 
constraints to design and implement a new method for computing stable 
models of {\em lparse} programs \cite{sns02}. 

\subsection{{\em Lparse} Programs}
\label{wa}

\citeA{sns02} introduced and studied an extension of normal logic
programming with weight atoms. Formally, a {\em weight atom} is an 
expression
\[
A = l[a_1=w_1,\ldots,a_k=w_k]u, 
\]
where $a_i$, $1\leq i\leq k$ are propositional atoms, and $l,u$ and 
$w_i$, $1\leq i\leq k$ are non-negative integers. If all weights $w_i$
are equal to 1, $A$ is a {\em cardinality atom}, written as $l\{a_1,
\ldots,a_k\}u$.

An {\em lparse rule} is an expression of the form
\[
A\leftarrow A_1,\ldots,A_n
\]
where $A$, $A_1,\ldots,A_n$ are weight atoms. We refer to sets of {\em 
lparse} rules as {\em lparse programs}. \citeA{sns02} defined for {\em
lparse} programs the semantics of stable models.

A set $M$ of atoms is a {\em model} of (or {\em satisfies}) a weight atom 
$l[a_1=w_1,\ldots, a_k=w_k]u$ if
\[
l\leq \sum_{i=1}^k \{w_i\colon a_i\in M\} \leq u.
\]

With this semantics a weight atom $l[a_1=w_1,\ldots, a_k= w_k]u$ can be
identified with a constraint $(X,C)$, where $X=\{a_1,\ldots,a_k\}$ and
\[
C=\{Y\subseteq X\colon l\leq \sum_{i=1}^k \{w_i\colon a_i\in Y\} \leq u\}.
\]

  We notice that all weights in a weight atom $W$ are non-negative.
  Therefore, if $M\subseteq M'\subseteq M''$ and both $M$ and $M''$
  are models of $W$, then $M'$ is also a model of $W$. It follows
  that the constraint $(X,C)$ we define above is convex.
  
Since $(X,C)$ is convex, weight atoms represent a class of convex
constraints and {\em lparse} programs syntactically are a class of 
programs with convex constraints. This relationship extends to 
the stable-model semantics. Namely, \citeA{mt04} and \citeA{mnt03,mnt06}
showed that {\em 
lparse} programs can be encoded as programs with monotone constraints so 
that the concept of a stable model is preserved. The transformation used
there coincides with the encoding $mc$ described in the previous section,
when we restrict the latter to {\em lparse} programs. Thus, we have the 
following theorem.

\begin{theorem}
\label{lparse}
Let $P$ be an lparse program. A set $M\subseteq \At$ is a stable model 
of $P$ according to the definition
by \citeA{sns02}
if and only if $M$ 
is a stable model of $P$ according to the definition given in the 
previous section (when $P$ is viewed as a convex-constraint 
program).
\end{theorem}

It follows that to compute stable models of {\em lparse} programs we 
can use the results obtained earlier in the paper, specifically the
results on program completion and loop formulas for convex-constraint
programs.

\noindent
{\bf Remark.}
To be precise, the syntax of {\em lparse} programs is more 
general. It allows both atoms and negated atoms to appear within weight 
atoms. It also allows weights to be negative. However, negative weights 
in {\em lparse} programs are treated just as a notational convenience. 
Specifically, an expression of the form $a=w$ within a weight atom (where 
$w<0$) represents the expression $\n(a)=-w$ (eliminating negative weights 
in this way from a weight atom requires modifications of the bounds 
associated with this weight atom). Moreover, by introducing new 
propositional variables one can remove occurrences of negative literals
from programs. These transformations preserve stable models (modulo
new atoms). \citeA{mt04} and \citeA{mnt03,mnt06} provide  
a detailed discussion
of this transformation.

In addition to weight atoms, the bodies of {\em lparse} rules may contain
propositional literals (atoms and negated atoms) as conjuncts. We can 
replace these propositional literals with weight atoms as follows: an 
atom $a$ can be replaced with the cardinality atom $1\{a\}$, and a
literal $\n(a)$ --- with the cardinality atom $\{a\}0$. This 
transformation preserves stable models, too. Moreover, the size of the 
resulting program does not increase more than by a constant factor. 
Thus, through the transformations discussed here, monotone- and 
convex-constraint programs capture arbitrary {\em lparse} programs.
\hfill$\Box$

\subsection{Computing Stable Models of {\em Lparse} Programs}

In this section we present an algorithm for computing stable models
of {\em lparse} programs. Our method uses the results we obtained in
Section \ref{secconvex} to reduce the problem to that of computing
models of the loop completion of an {\em lparse} program. The loop
completion is a formula in the logic $\plcc$, in which the class of
convex constraints is restricted to weight constraints, as defined
in the previous subsection. We will denote the fragment of the logic
$\plcc$ consisting of such formulas by $\plwc$.

To make the method practical, we need programs to compute models of 
theories in the logic $\plwc$. We will now show a general way to 
adapt to this task off-the-shelf {\em pseudo-boolean constraint 
solvers}
\cite{es03,arms02,wal97,pbcomp05,lt03}.

{\em Pseudo-boolean constraints} ($\pb$ for short) are integer
programming constraints
in which variables have 0-1 domains. We will write them as inequalities 
\begin{equation}
\label{pbeq}
w_1\times x_1 + \ldots + w_k\times x_k \cmp w,
\end{equation}
where $\cmp$ stands for one of the relations $\leq$, $\geq$, $<$ and 
$>$, $w_i$'s and $w$ are integer coefficients (not necessarily 
non-negative), and $x_i$'s are integers taking value 0 or 1. A set of 
pseudo-boolean constraints is a {\em pseudo-boolean theory}. 

Pseudo-boolean constraints can be viewed as constraints. The basic idea 
is to treat each 0-1 variable $x$ as a propositional atom (which we will
denote by the same letter). Under this 
correspondence, a pseudo-boolean constraint (\ref{pbeq}) is equivalent
to the constraint $(X,C)$, where $X =\{x_1,\ldots,x_k\}$ and
\[
C=\{Y\subseteq X\colon \sum_{i=1}^k\{w_i\colon x_i\in Y\} \cmp w\}
\]
in the sense that solutions to (\ref{pbeq}) correspond to models of 
$(X,C)$ ($x_i=1$ in a solution if and only if $x_i$ is true in the 
corresponding model). In particular, if all coefficients $w_i$ and the 
bound $w$ in (\ref{pbeq}) are non-negative, and if $\cmp=\mbox{
`$\geq$'}$, then the constraint (\ref{pbeq}) is equivalent to a monotone 
lower-bound weight atom $w[x_1=w_1,\ldots,x_n=w_n]$. 

It follows that an arbitrary weight atom can be represented by one or 
two pseudo-boolean constraints. More generally, an arbitrary $\plwc$ 
formula $F$ can be encoded as a set of $\pb$ constraints. We will describe
the translation as a two-step process.

The first step consists of converting $F$ to a {\em
clausal} form $\tocl(F)$\footnote{A $\plwc$ {\em clause} is any formula 
$B_1\wedge \ldots \wedge B_m \rightarrow H_1\vee\ldots\vee H_n$, where 
$B_i$ and $H_j$ are weight atoms.}. To control the size of the 
translation, we introduce auxiliary propositional atoms. Below, we 
describe the translation $F \mapsto \tocl(F)$ under the assumption 
that $F$ is a formula of the loop completion of an {\em lparse} program 
$P$. Our main motivation is to compute stable models of logic programs
and to this end algorithms for computing models of loop completions
are sufficient.

Let $F$ be a formula in the loop completion of an {\em lparse}-program 
$P$. We define $\tocl(F)$ as follows (in the transformation, we use a 
propositional atom $x$ as a shorthand for the cardinality atom $C(x)=
1\{x\}$).

\noindent
1. If $F$ is of the form $A_1 \wedge \ldots \wedge A_n \rightarrow A$,
         then $\tocl(F)=F$\\
2. If $F$ is of the form $ x \rightarrow ([\bd(r_1)]^{\wedge}) \vee
         \ldots \vee ([\bd(r_l)]^{\wedge})$,
         then we introduce new propositional atoms $b_{r,1},\ldots,b_{r,l}$ and
         set $\tocl(F)$ to consist of the following $\plwc$ clauses:
         \[
           x \rightarrow b_{r,1} \vee \ldots \vee b_{r,l}
         \]
         \[
           [\bd(r_i)]^{\wedge} \rightarrow b_{r,i}, \textrm{ for every }\bd(r_i)
         \]
         \[
           b_{r,i} \rightarrow A_j,\textrm{ for every }\bd(r_i)\textrm{ and }
           A_j\in \bd(r_i)
         \]
3. If $F$ is of the form $\bigvee L \rightarrow \bigvee_r
       \{\bdf_L(r)\}$,
         where $L$ is a set of atoms, and every $\bdf_L(r)$ is a conjunction of
         weight atoms, then we introduce new propositional atoms $bdf_{L,r}$ for
         every $\bdf_L(r)$ in $F$ and represent $\bigvee L$ as the weight atom
         $W_L=1[l_i=1:l_i\in L]$. We then define $\tocl(F)$ to consist of
         the following clauses:
         \[
           W_L \rightarrow \bigvee bdf_{L,r}
         \]
         \[
           \bdf_L(r) \rightarrow bdf_{L,r},\textrm{ for every }\bdf_L(r)\in F
         \]
         \[
           bdf_{L,r} \rightarrow A_j,\textrm{ for every }\bdf_L(r)\in F
           \textrm{ and }A_j\in \bdf_L(r).
         \]

It is clear that the size $\tocl(F)$ is linear in the size of $F$.

The second step of the translation, converts a $\plwc$ formula $C$ in
a clausal form into a set of $\pb$ constraints, $\topb(C)$.
To define the translation $C\rightarrow \topb(C)$, let us consider a 
$\plwc$ clause $C$ of the form
\begin{equation}
\label{clause}
B_1 \wedge \ldots \wedge B_m \rightarrow H_1 \vee \ldots \vee H_n,
\end{equation}
where $B_i$'s and $H_i$'s are weight atoms.

We introduce new propositional atoms $b_1,\ldots,b_m$ and $h_1,\ldots,
h_n$ to represent each weight atom in the clause. As noted earlier in
the paper, we simply write $x$ for a weight atoms of the form $1[x=1]$.
With the new atoms, the clause (\ref{clause}) becomes a propositional 
clause $b_1\wedge \ldots \wedge b_m \rightarrow h_1\vee \ldots \vee 
h_n$. We represent it by the following $\pb$ constraint:
\begin{equation}
\label{pbclause}
- b_1 - \ldots - b_m + h_1 + \ldots + h_n \geq 1 - m.
\end{equation}
Here and later in the paper, we use the same symbols to denote
propositional variables and the corresponding 0-1 integer variables.
The context will always imply the correct meaning of the symbols.
Under this convention, it is easy to see that a propositional clause 
$b_1\wedge \ldots \wedge b_m \rightarrow h_1\vee \ldots \vee h_n$ and
its $\pb$ constraint (\ref{pbclause}) have the same models.

We introduce next $\pb$ constraints that enforce the equivalence of the
newly introduced atoms $b_i$ (or $h_i$) and the corresponding weight 
atoms $B_i$ (or $H_i$).

Let $B=l[a_1=w_1,\ldots,a_k=w_k]u$ be a weight atom and $b$ a 
propositional atom. We split $B$ to $B^+$ and $B^-$ and introduce two
more atoms $b^+$ and $b^-$. To model $B\equiv b$, we model with 
pseudo-boolean constraints the following three equivalences:
$b\equiv b^+ \wedge b^-$, $b^+ \equiv B^+$, and $b^- \equiv B^-$.

\noindent
1. The first equivalence can be captured with three propositional 
clauses. Hence the following three $\pb$ constraints model that 
equivalence:
\begin{equation}
      \label{pb3eq1}
      -b + b^+ \geq 0
\end{equation}
\begin{equation}
      \label{pb3eq2}
      -b + b^- \geq 0
\end{equation}
\begin{equation}
      \label{pb3eq3}
      -b^+ - b^- + b \geq -1
\end{equation}
2. The second equivalence, $b^+\equiv B^+$, can be modeled by the 
following two $\pb$ constraints
        \begin{equation}
          \label{pblb1}
          (-l)\times b^+ + \sum_{i=1}^k(a_i\times w_i) \geq 0
        \end{equation}
        \begin{equation}
          \label{pblb2}
          -(\sum_{i=1}^k w_i -l + 1)\times b^+ +
	  \sum_{i=1}^k(a_i\times w_i) \leq l-1
        \end{equation}
3. Similarly, the third equivalence, $b^-\equiv B^-$, can be modeled 
by the following two $\pb$ constraints
        \begin{equation}
          \label{pbub1}
          (\sum_{i=1}^k w_i - u)\times b^- +
	  \sum_{i=1}^k(a_i\times w_i) \leq \sum_{i=1}^k w_i
        \end{equation}
        \begin{equation}
          \label{pbub2}
          (u + 1)\times b^- + \sum_{i=1}^k(a_i\times w_i) \geq u+1
        \end{equation}

We define now $\topb(C)$, for a $\plwc$ clause $C$, as the set of all
pseudo-boolean constraints (\ref{pbclause}) and (\ref{pb3eq1}), 
(\ref{pb3eq2}), (\ref{pb3eq3}), (\ref{pbub1}), (\ref{pbub2}), 
(\ref{pblb1}), (\ref{pblb2}) constructed for every weight atom occurring
in $C$. One can verify that the size of $\topb(C)$ is linear in the
size of $C$. Therefore, $\topb(\tocl(F))$ has size linear in the size of
$F$.

In the special case where all $B_i$'s and $H_j$'s are weight atoms of
the form $1[b_i=1]$ and $1[h_j=1]$, we do not need to introduce any
new atoms and $\pb$ constraints (\ref{pb3eq1}), (\ref{pb3eq2}),
(\ref{pb3eq3}), (\ref{pbub1}), (\ref{pbub2}), (\ref{pblb1}),
(\ref{pblb2}). Then $\topb(C)$ consists of a single $\pb$ constraint
(\ref{pbclause}).

We have the following theorem establishing the correctness of the
transformation $\tau$. The proof of the theorem is straightforward.

\begin{theorem}
     Let $F$ be a loop completion formula in logic $\plwc$, and $M$ a
     set of atoms, $M\subseteq \At(F)$. Then $M$ is a model of $F$ in
     $\plwc$ logic if and only if $M$ has a unique extension $M'$ by
     some of the new atoms in $\At(\topb(\tocl(F)))$ such that $M'$
     is a model of the pseudo-boolean theory $\topb(\tocl(F))$. 
\end{theorem}

We note that when we use solvers designed for $\plwc$ theories, then
translation $\topb$ is no longer needed. The benefit of using such
solvers is that we do not need to split weight atoms in the $\plwc$
theories and do not need the auxiliary atoms introduced in $\topb$.

\subsubsection{The Algorithm}

We follow the approach proposed
by \citeA{lz02}.
As in that paper, we 
first compute the completion of a {\em lparse} program. Then, we iteratively 
compute models of the completion using a $\pb$ solver. Whenever a 
model is found, we test it for stability. If the model is not a 
stable model of the program, we extend the completion by loop formulas
identified in Theorem \ref{convexloop.cor}. Often, adding a single loop
formula filters out several models of $\comp(P)$ that are not stable
models of $P$.

The results given in the previous section ensure that our algorithm is 
correct. We present it in Figure \ref{fig.alg}. We note that it may 
happen that in the worst case exponentially many loop formulas are 
needed before the first stable model is found or we determine that no 
stable models exist \cite{lz02}. However, that problem arises only rarely in 
practical situations\footnote{In fact, in many cases programs turn out
to be tight with respect to their supported models. Therefore, supported 
models are stable and no loop formulas are necessary at all.}.

\begin{figure}
\noindent
\rule{12.2cm}{0.5mm}\\
Input: $P$ --- a {\em lparse} program;\\
\hspace*{0.4in} $A$ --- a pseudo-boolean solver

\noindent
{\bf BEGIN}\\
\mbox{}\ \ \ compute the completion $\comp(P)$ of $P$;\\
\mbox{}\ \ \ $T$ := $\topb(\tocl(\comp(P)))$;\\
\mbox{}\ \ \ {\bf do}\\
\mbox{}\ \ \ \ \ \ {\bf if} (solver $A$ finds no models of $T$)\\
\mbox{}\ \ \ \ \ \ \ \ \ \ \ output ``no stable models found'' and terminate;\\ 
\mbox{}\ \ \ \ \ \ \ $M$ := a model of $T$ found by $A$;\\
\mbox{}\ \ \ \ \ \ \ {\bf if} ($M$ is stable) output $M$ and terminate;\\
\mbox{}\ \ \ \ \ \ \ compute the reduct $P^M$ of $P$ with respect to $M$;\\
\mbox{}\ \ \ \ \ \ \ compute the greatest stable model $M'$, contained in
$M$, of $P^M$;\\
\mbox{}\ \ \ \ \ \ \ $M^-$ := $M\setminus M'$;\\
\mbox{}\ \ \ \ \ \ \ find all terminating loops in $M^-$;\\
\mbox{}\ \ \ \ \ \ \ compute loop formulas and convert them into $\pb$
constraints using\\
\mbox{}\ \ \ \ \ \ \ \ \ \ \ $\topb$ and $\tocl$;\\
\mbox{}\ \ \ \ \ \ \ add all $\pb$ constraints computed in the previous
step to $T$;\\
\mbox{}\ \ \ {\bf while} ({\bf true});\\
{\bf END}\\
\rule{12.2cm}{0.5mm}\\
\caption{Algorithm of $\pmd$}
\label{fig.alg}
\end{figure}

The implementation of $\pmd$ supports several $\pb$ solvers such as
{\em satzoo} \cite{es03}, {\em pbs} \cite{arms02}, {\em wsatoip}
\cite{wal97}. It also supports a program {\em wsatcc} \cite{lt03} for 
computing models of $\plwc$ theories. When this last program is used,
the transformation, from ``clausal'' $\plwc$ theories to pseudo-boolean
theories is not needed. The first two of these four programs are 
complete $\pb$ solvers. The latter two are local-search solvers based 
on {\em wsat} \cite{skc94}.

We output the message ``no stable model found'' in the first line
of the loop and not simply ``no stable models exist'' since in the case 
when $A$ is a local-search algorithm, failure to find a model of the 
completion (extended with loop formulas in iteration two and the 
subsequent ones) does not imply that no models exist.

\subsection{Performance}

In this section, we present experimental results concerning the
performance of $\pmd$. The experiments compared $\pmd$, combined with 
several $\pb$ solvers, to $\smd$ \cite{sns02} and $\cmd$ \cite{cmodels}.
We focused our experiments on problems whose statements explicitly involve 
pseudo-boolean constraints, as we designed $\pmd$ with such problems in
mind.

For most benchmark problems we tried $\cmd$ did not perform well. 
Only in one case (vertex-cover benchmark) the performance of $\cmd$ was 
competitive, although even in this case it was not the best performer. 
Therefore, we do not report here results we compiled for $\cmd$. For a 
complete set of results we obtained in the experiments we refer to 
\url{http://www.cs.uky.edu/ai/pbmodels}.

In the experiments we used instances of the following problems: {\em 
traveling salesperson}, {\em weighted $n$-queens}, {\em weighted Latin 
square}, {\em magic square}, {\em vertex cover}, and {\em Towers of 
Hanoi}. The {\em lparse} programs we used for the first four problems 
involve general pseudo-boolean constraints. Programs modeling the 
last two problems contain cardinality constraints only.

\noindent
{\bf Traveling salesperson problem (TSP)}. An instance consists of a 
weighted complete graph with $n$ vertices, and a bound $w$. All edge 
weights and $w$ are non-negative integers. A solution to an instance is 
a Hamiltonian cycle whose total weight (the sum of the weights of all 
its edges) is less than or equal to $w$.

We randomly generated $50$ weighted complete graphs with $20$ vertices,
To this end, in each case we assign to every edge of a complete
undirected
graph an integer weight selected uniformly at random from the 
range $[1..19]$.  By setting $w$ to $100$ we obtained a set of ``easy'' 
instances, denoted by {\em TSP-e} (the bound is high enough for every 
instance in the set to have a solution). From the same collection of 
graphs, we also created a set of ``hard'' instances, denoted by {\em 
TSP-h}, by setting $w$ to $62$. Since the requirement on the total weight 
is stronger, the instances in this set in general take more time.

\noindent
{\bf Weighted $n$-queens problem (WNQ)}. An instance to the problem 
consists of a weighted $n\times n$ chess board and a bound $w$. All 
weights and the bound are non-negative integers. A solution to an 
instance is a placement of $n$ queens on the chess board so that no two
queens attack each other and the weight of the placement (the sum of the 
weights of the squares with queens) is not greater than $w$.

We randomly generated $50$ weighted chess boards of the size $20\times 
20$, where each chess board is represented by a set of $n\times n$ 
integer weights $w_{i,j}$, $1\leq i,j\leq n$, all selected uniformly at 
random from the range $[1..19]$. We then created two sets of instances, 
easy (denoted by {\em wnq-e}) and hard (denoted by {\em wnq-h}), by 
setting the bound $w$ to 70 and 50, respectively.

\noindent
{\bf Weighted Latin square problem (WLSQ)}. An instance consists of an 
$n\times n$ array of weights $w_{i,j}$, and a bound $w$. All weights
$w_{i,j}$ and $w$ are non-negative integers. A solution to an instance 
is an $n\times n$ array $L$ with all entries from $\{1,\ldots,n\}$ and 
such that each element in $\{1,\ldots,n\}$ occurs exactly once in each 
row and in each column of $L$, and $\sum_{i=1}^n\sum_{j=1}^n L[i,j]
\times w_{i,j} \leq w$.

We set $n=10$ and we randomly generated $50$ sets of integer weights, 
selecting them uniformly at random from the range $[1..9]$. Again we 
created two families of instances, easy ({\em wlsq-e}) and hard ({\em
wlsq-h}), by setting $w$ to $280$ and $225$, respectively.

\noindent
{\bf Magic square problem}. An instance consists of a positive integer 
$n$. The goal is to construct an $n\times n$ array using each integer 
$1,\ldots n^2$ as an entry in the array exactly once in such a way that
entries in each row, each column and in both main diagonals sum up to
$n(n^2+1)/2$. For the experiments we used the magic square problem for 
$n=4,5$ and $6$.

\noindent
{\bf Vertex cover problem}. An instance consists of graph with $n$ 
vertices and $m$ edges, and a non-negative integer $k$ --- a bound. A 
solution to the instance is a subset of vertices of the graph with no
more than $k$ vertices and such that at least one end vertex of every
edge in the graph is in the subset.

We randomly generated $50$ graphs, each with $80$ vertices and $400$ 
edges. For each graph, we set $k$ to be a smallest integer such that
a vertex cover with that many elements still exists. 

\noindent
{\bf Towers of Hanoi problem}. This is a slight generalization of the 
original problem. We considered the case with six disks
and three pegs.
An instance 
consists of an initial configuration of disks that satisfies the 
constraint of the problem (larger disk must not be on top of a smaller 
one) but does not necessarily requires that all disks are on one peg. 
These initial configurations were selected so that they were 31,
36, 41 and 63 steps away from the goal configuration (all disks from the
largest to the smallest on the third peg), respectively. We also 
considered a standard version of the problem with seven disks, in which 
the initial configuration is $127$ steps away from the goal.

We encoded each of these problems as a program in the general syntax of 
{\em lparse}, which allows the use of relation symbols and variables
\cite{syr99a}. The programs are available at 
\url{http://www.cs.uky.edu/ai/pbmodels}. We then used these programs in 
combination with appropriate instances as inputs to {\em lparse} 
\cite{syr99a}. In this way, for each problem and each set of instances 
we generated a 
family of ground (propositional) {\em lparse} programs so that stable 
models of each of these programs represent solutions to the corresponding 
instances of the problem (if there are no stable models, there are no 
solutions). We used these families of {\em lparse} programs as inputs 
to solvers we were testing. All these ground programs are also available
at \url{http://www.cs.uky.edu/ai/pbmodels}.

In the tests, we used $\pmd$ with the following four $\pb$ solvers: 
$\satzoo$ \cite{es03}, $\pbs$ \cite{arms02}, $\wsatcc$ \cite{lt03}, 
and $\wsatoip$ \cite{wal97}. In particular, $\wsatcc$ deals with
$\plwc$ theories directly.

All experiments were run on machines with 3.2GHz Pentium 4 CPU, 1GB 
memory, running Linux with kernel version 2.6.11, gcc version 3.3.4.
In all cases, we used 1000 seconds as the timeout limit.

We first show the results for the {\em magic square} and {\em towers of
Hanoi} problems. In Table \ref{tab.msth}, for each solver and each 
instance, we report the corresponding running time in seconds.
Local-search solvers were unable to solve any of the instances
in the two problems and so are not included in the table.

\begin{table}[ht]
\begin{footnotesize}
\begin{center}
\begin{tabular}{|c|c|c|c|}
\hline
\emph{Benchmark}&
\emph{smodels}&
\emph{pbmodels-satzoo}&
\emph{pbmodels-pbs}\tabularnewline
\hline
\hline
\emph{magic square $(4\times4)$}&
$1.36$&
$1.70$&
$2.41$\tabularnewline
\hline
\emph{magic square $(5\times5)$}&
$>1000$&
$28.13$&
$0.31$\tabularnewline
\hline
\emph{magic square $(6\times6)$}&
$>1000$&
$75.58$&
$>1000$\tabularnewline
\hline
\emph{towers of Hanoi $(d=6,t=31)$}&
$16.19$&
$18.47$&
$1.44$\tabularnewline
\hline
\emph{towers of Hanoi $(d=6,t=36)$}&
$32.21$&
$31.72$&
$1.54$\tabularnewline
\hline
\emph{towers of Hanoi $(d=6,t=41)$}&
$296.32$&
$49.90$&
$3.12$\tabularnewline
\hline
\emph{towers of Hanoi $(d=6,t=63)$}&
$>1000$&
$>1000$&
$3.67$\tabularnewline
\hline
\emph{towers of Hanoi $(d=7,t=127)$}&
$>1000$&
$>1000$&
$22.83$\tabularnewline
\hline
\end{tabular}
\end{center}
\caption{Magic square and towers of Hanoi problems}
\label{tab.msth}
\end{footnotesize}
\end{table}

Both $\pmdzoo$ and $\pmdpbs$ perform better than $\smd$ on programs
obtained from the instances of both problems. We observe that
$\pmdpbs$ performs exceptionally well in the tower of Hanoi problem. It 
is the only solver that can compute a plan for $7$ disks, which requires
127 steps. Magic square and Towers of Hanoi problems are highly regular.
Such problems are often a challenge for local-search problems, which 
may explain a poor performance we observed for $pbmodels\mbox{-}wsatcc$ and
$\pmdoip$ on these two benchmarks.

For the remaining four problems, we used 50-element families of 
instances, which we generated randomly in the way discussed above. We
studied the performance of complete solvers ($\smd$, $\pmdzoo$ and
$\pmdpbs$) on all instances. We then included local-search solvers
($\pmdwcc$ and $pbmodels\-wsatoip$) in the comparisons but restricted attention
only to instances that were determined to be satisfiable (as
local-search solvers are, by their design, unable to decide
unsatisfiability). In Table \ref{tab.suminst}, for each family we list 
how many of its instances are satisfiable, unsatisfiable, and for how 
many of the instances none of the solvers we tried was able to decide 
satisfiability.

\begin{table}
\begin{footnotesize}
\begin{center}
\begin{tabular}{|c|c|c|c|}
\hline
&
\emph{\# of SAT instances}&
\emph{\# of UNSAT instances}&
\emph{\# of UNKNOWN instances}\tabularnewline
\hline
\hline
\emph{TSP-e}&
$50$&
$0$&
$0$\tabularnewline
\hline
\emph{TSP-h}&
$31$&
$1$&
$18$\tabularnewline
\hline
\emph{wnq-e}&
$49$&
$0$&
$1$\tabularnewline
\hline
\emph{wnq-h}&
$29$&
$0$&
$21$\tabularnewline
\hline
\emph{wlsq-e}&
$45$&
$4$&
$1$\tabularnewline
\hline
\emph{wlsq-h}&
$8$&
$41$&
$1$\tabularnewline
\hline
\emph{vtxcov }&
$50$&
$0$&
$0$\tabularnewline
\hline
\end{tabular}
\end{center}
\caption{Summary of Instances}
\label{tab.suminst}
\end{footnotesize}
\end{table}

In Table \ref{tab.sum}, for each of the seven families of instances
and for each {\em complete} solver, we report two values $s/w$, 
where $s$ is the number of instances solved by the solver and $w$ is 
the number of times it was the fastest among the three.

\begin{table}
\begin{footnotesize}
\begin{center}
\begin{tabular}{|c|c|c|c|c|}
\hline
&
\emph{smodels}&
\emph{pbmodels-satzoo}&
\emph{pbmodels-pbs}\tabularnewline
\hline
\hline
\emph{TSP-e}&
$45/17$&
$50/30$&
$18/3$\tabularnewline
\hline
\emph{TSP-h}&
$7/3$&
$16/14$&
$0/0$\tabularnewline
\hline
\emph{wnq-e}&
$11/5$&
$26/23$&
$0/0$\tabularnewline
\hline
\emph{wnq-h}&
$2/2$&
$0/0$&
$0/0$\tabularnewline
\hline
\emph{wlsq-e}&
$21/1$&
$49/29$&
$46/19$\tabularnewline
\hline
\emph{wlsq-h}&
$0/0$&
$47/26$&
$47/23$\tabularnewline
\hline
\emph{vtxcov }&
$50/40$&
$50/1$&
$47/3$\tabularnewline
\hline
\emph{sum over all}&
$136/68$&
$238/123$&
$158/48$\tabularnewline
\hline
\end{tabular}
\end{center}
\caption{Summary on all instances}
\label{tab.sum}
\end{footnotesize}
\end{table}

The results in Table \ref{tab.sum} show that overall $\pmdzoo$ solved
more instances than $pbmodels\mbox{-}$ $pbs$, followed by $\smd$. When we look at the 
number of times a solver was the fastest one, $\pmdzoo$ was a clear
winner overall, followed by $\smd$ and then by $\pmdpbs$. Looking
at the seven families of tests individually, we see that $\pmdzoo$
performed better than the other two solvers on five of the families. 
On the other two $\smd$ was the best performer (although, it is a clear 
winner only on the vertex-cover benchmark; all solvers were essentially
ineffective on the \emph{wnq-h}).

We also studied the performance of $\pmd$ combined with local-search 
solvers $\wsatcc$ \cite{lt03} and $\wsatoip$ \cite{wal97}. For this
study, we considered only those instances in the seven families that
we knew were satisfiable. Table \ref{tab.sumsat} presents results for all
solvers we studied (including the complete ones). As before, each entry
provides a pair of numbers $s/w$, where $s$ is the number of solved
instances and $w$ is the number of times the solver performed better
than its competitors.

\begin{table}
\begin{footnotesize}
\begin{center}
\begin{tabular}{|c|c|c|c|c|c|c|}
\hline
&
\emph{smodels}&
\emph{pbmd-satzoo}&
\emph{pbmd-pbs}&
\emph{pbmd-wsatcc}&
\emph{pbmd-wsatoip}\tabularnewline
\hline
\hline
\emph{TSP-e}&
$45/3$&
$50/5$&
$18/2$&
$32/7$&
$47/34$\tabularnewline
\hline
\emph{TSP-h}&
$7/0$&
$16/2$&
$0/0$&
$19/6$&
$28/22$\tabularnewline
\hline
\emph{wnq-e}&
$11/0$&
$26/0$&
$0/0$&
$49/45$&
$49/4$\tabularnewline
\hline
\emph{wnq-h}&
$2/0$&
$0/0$&
$0/0$&
$29/15$&
$29/14$\tabularnewline
\hline
\emph{wlsq-e}&
$21/0$&
$45/0$&
$44/0$&
$45/33$&
$45/14$\tabularnewline
\hline
\emph{wlsq-h}&
$0/0$&
$7/0$&
$8/0$&
$7/1$&
$8/7$\tabularnewline
\hline
\emph{vtxcov }&
$50/0$&
$50/0$&
$47/0$&
$50/36$&
$50/15$\tabularnewline
\hline
\emph{sum over all}&
$136/3$&
$194/7$&
$117/2$&
$231/143$&
$256/110$\tabularnewline
\hline
\end{tabular}
\end{center}
\caption{Summary on SAT instances}
\label{tab.sumsat}
\end{footnotesize}
\end{table}

The results show superior performance of $\pmd$ combined with
local-search solvers. They solve more instances than complete
solvers (including $\smd$). In addition, they are significantly faster,
winning much more frequently than complete solvers do (complete solvers
were faster only on 12 instances, while local-search solvers were
faster on 253 instances).

Our results demonstrate that $\pmd$ with solvers of pseudo-boolean
constraints outperforms $\smd$ on several types of search problems
involving pseudo-boolean (weight) constraints).

We note that we also analyzed the run-time distributions for each of 
these families of instances. A run-time distribution is regarded as 
a more accurate and detailed measure of the performance of algorithms
on randomly generated instances\footnote{\citeA{hs05} provide a detailed
discussion of this matter in the context of local-search methods.}. The
results are consistent with the summary results presented above and 
confirm our conclusions. As the discussion of run-time distributions 
requires much space, we do not include this analysis here.
They are available at the website \url{http://www.cs.uky.edu/ai/pbmodels}.

\section{Related work}

Extensions of logic programming with means to model properties of {\em
sets} (typically consisting of ground terms) have been extensively 
studied. Usually, these extensions are referred to by the common term 
of {\em logic programming with aggregates}. The term comes from the fact
that most properties of sets of practical interest are defined through 
``aggregate'' operations such as sum, count, maximum, minimum and 
average. We chose the term {\em constraint} to stress that we speak 
about abstract properties that define constraints on truth assignments
(which we view as sets of atoms).

\citeA{mum90}, and \citeA{ks91} were among the first to study
logic programs with aggregates. Recently, \citeA{nss99} and 
\citeA{sns02} 
introduced the class of {\em lparse} programs. We discussed this 
formalism in detail earlier in this paper. 

\citeA{p04} and \citeA{pdb06} studied a more general class of aggregates 
and developed 
a systematic theory of aggregates in logic programming based on the 
approximation theory \cite{dmt00a}. The resulting theory covers not only 
the stable models semantics but also the supported-model semantics and 
extensions of 3-valued Kripke-Kleene and well-founded semantics. The 
formalism introduced and studied by \citeA{p04} and \citeA{pdb06} 
allows for arbitrary aggregates (not only monotone ones) to appear in 
the bodies of rules. However, it does not 
allow for aggregates to appear in the heads of program clauses. Due to 
differences in the syntax and the scope of semantics studied there is no 
simple way to relate
\citeS{p04} and \citeS{pdb06} formalism
to programs with monotone (convex) 
constraints. We note though that programs with abstract monotone
constraints with the heads of rules of the form $C(a)$ can be viewed
almost literally as programs in the formalism by \citeA{p04} and 
\citeA{pdb06} and that they have the same stable models according to 
the definitions we used in this paper and those by \citeA{p04}
and \citeA{pdb06}.

\citeA{flp04} developed the theory of {\em disjunctive} logic programs 
with aggregates. Similarly as \citeA{p04} and \citeA{pdb06},
\citeA{flp04} do not allow for aggregates to appear in the heads of 
program clauses. This is one 
of the differences between that approach and programs with monotone 
(convex) constraints we studied here. The other major difference is 
related to the postulate of the minimality of stable models (called {\em 
answer sets} in the context of the formalism considered
by \citeR{flp04}).
In keeping with the spirit of the original answer-set semantics 
\cite{gl90b}, answer sets of disjunctive programs with aggregates, as 
defined
by \citeA{flp04},
are minimal models. Stable models of programs 
with abstract constraints do not have this property. However, for the
class of programs with abstract monotone constraints with the heads of 
rules of the form $C(a)$ the semantics of answer sets defined
by \citeA{flp04}
coincides with the semantics of stable models by
\citeA{mt04} and \citeA{mnt03,mnt06}.

Yet another approach to aggregates in logic programming was presented
by \citeA{sp06}.
That approach considered programs of the syntax similar 
to programs with monotone abstract constraints. It allowed arbitrary
constraints (not only monotone ones) but not under the scope of $\n$
operator. A general principle behind the definition of the stable-model 
semantics
by \citeA{sp06}
is to view a program with constraints 
as a concise representation of a set of its ``instances'', each being a 
normal logic program. Stable models of the program with constraints are 
defined as stable models of its instances and is quite different from 
the operator-based definition
by \citeA{mt04} and \citeA{mnt03,mnt06}.
However, for programs
with {\em monotone} constraint atoms which fall in the scope of the 
formalism of \citeA{sp06} both approaches coincide. 

We also note that recently \citeA{spt06} presented a {\em conservative} 
extension of the syntax proposed by \citeA{mt04} \citeA{mnt06}, in which 
clauses are built of arbitrary constraint atoms.

Finally, we point out the work by
\citeA{fl04} and \citeA{fer05}
which treats 
aggregates as {\em nested expressions}. In particular, \citeA{fer05} 
introduces a propositional logic with a certain nonclassical semantics, 
and shows that it extends several approaches to programs with aggregates, 
including those by
\citeA{sns02}
(restricted to core lparse programs) and 
\citeA{flp04}.
The nature of the relationship of the formalism
by \citeA{fer05}
and programs with abstract constraints remains an open
problem.

\section{Conclusions}

Our work shows that concepts, techniques and results from normal logic
programming, concerning strong and uniform equivalence, tightness and
Fages lemma, program completion and loop formulas, generalize to the 
abstract setting of programs with monotone and convex constraints.
These general properties specialize to {\em new} results about {\em 
lparse} programs (with the exception of the characterization strong 
equivalence of {\em lparse} programs, which was first obtained
by \citeR{tu03}).

Given these results we implemented a new software {\em pbmodels} for 
computing stable models of {\em lparse} programs. The approach reduces
the problem to that of computing models of theories consisting of 
pseudo-boolean constraints, for which several fast solvers exist
\cite{pbcomp05}. Our experimental results show that {\em pbmodels} 
with $\pb$ solvers, especially local search $\pb$ solvers, performs better
than $\smd$ on several types of search problems we tested. Moreover, as
new and more efficient solvers of pseudo-boolean constraints become
available (the problem is receiving much attention in the satisfiability
and integer programming communities), the performance of {\em pbmodels}
will improve accordingly.

\section*{Acknowledgments}

We acknowledge the support of NSF grants IIS-0097278 and IIS-0325063.
We are grateful to the reviewers for their useful comments and
suggestions.

This paper combines and extends results included in conference papers
\cite{lt05,lt05b}.

{\small


\begin{thebibliography}{}

\bibitem[\protect\BCAY{Aloul, Ramani, Markov,\ \BBA\ Sakallah}{Aloul
  et~al.}{2002}]{arms02}
Aloul, F., Ramani, A., Markov, I., \BBA\ Sakallah, K. \BBOP2002\BBCP.
\newblock \BBOQ {PBS}: a backtrack-search pseudo-boolean solver and
  optimizer\BBCQ\
\newblock In {\Bem {Proceedings of the 5th International Symposium on Theory
  and Applications of Satisfiability}}, \BPGS\ 346 -- 353.

\bibitem[\protect\BCAY{Babovich\ \BBA\ Lifschitz}{Babovich\ \BBA\
  Lifschitz}{2002}]{cmodels}
Babovich, Y.\BBACOMMA\  \BBA\ Lifschitz, V. \BBOP2002\BBCP.
\newblock \BBOQ {\em Cmodels package}\BBCQ.
\newblock \url{http://www.cs.utexas.edu/users/tag/cmodels.html}.

\bibitem[\protect\BCAY{Baral}{Baral}{2003}]{baral03}
Baral, C. \BBOP2003\BBCP.
\newblock {\Bem {Knowledge representation, reasoning and declarative problem
  solving}}.
\newblock {Cambridge University Press}.

\bibitem[\protect\BCAY{Clark}{Clark}{1978}]{cl78}
Clark, K. \BBOP1978\BBCP.
\newblock \BBOQ Negation as failure\BBCQ\
\newblock In Gallaire, H.\BBACOMMA\  \BBA\ Minker, J.\BEDS, {\Bem Logic and
  data bases}, \BPGS\ 293--322. Plenum Press, New York-London.

\bibitem[\protect\BCAY{{Dell'Armi}, Faber, Ielpa, Leone,\ \BBA\
  Pfeifer}{{Dell'Armi} et~al.}{2003}]{dlv-agg-03}
{Dell'Armi}, T., Faber, W., Ielpa, G., Leone, N., \BBA\ Pfeifer, G.
  \BBOP2003\BBCP.
\newblock \BBOQ Aggregate functions in disjunctive logic programming:
  semantics, complexity, and implementation in {DLV}\BBCQ\
\newblock In {\Bem Proceedings of the 18th International Joint Conference on
  Artificial Intelligence (IJCAI-2003)}, \BPGS\ 847--852. Morgan Kaufmann.

\bibitem[\protect\BCAY{Denecker, Marek,\ \BBA\ Truszczy{\'n}ski}{Denecker
  et~al.}{2000}]{dmt00a}
Denecker, M., Marek, V., \BBA\ Truszczy{\'n}ski, M. \BBOP2000\BBCP.
\newblock \BBOQ Approximations, stable operators, well-founded fixpoints and
  applications in nonmonotonic reasoning\BBCQ\
\newblock In Minker, J.\BED, {\Bem Logic-Based Artificial Intelligence}, \BPGS\
  127--144. Kluwer Academic Publishers.

\bibitem[\protect\BCAY{Denecker, Pelov,\ \BBA\ Bruy\-nooghe}{Denecker
  et~al.}{2001}]{dpb01}
Denecker, M., Pelov, N., \BBA\ Bruy\-nooghe, M. \BBOP2001\BBCP.
\newblock \BBOQ Ultimate well-founded and stable semantics for logic programs
  with aggregates\BBCQ\
\newblock In Codognet, P.\BED, {\Bem Logic programming, Proceedings of the 2001
  International Conference on Logic Programming}, \lowercase{\BVOL}\ 2237,
  \BPGS\ 212--226. Springer.

\bibitem[\protect\BCAY{E{\'e}n\ \BBA\ S{\"o}rensson}{E{\'e}n\ \BBA\
  S{\"o}rensson}{2003}]{es03}
E{\'e}n, N.\BBACOMMA\  \BBA\ S{\"o}rensson, N. \BBOP2003\BBCP.
\newblock \BBOQ An extensible {SAT} solver\BBCQ\
\newblock In {\Bem Theory and Applications of Satisfiability Testing, 6th
  International Conference, SAT-2003}, \lowercase{\BVOL}\ 2919 of {\Bem LNCS},
  \BPGS\ 502--518. Springer.

\bibitem[\protect\BCAY{Eiter\ \BBA\ Fink}{Eiter\ \BBA\ Fink}{2003}]{ef03}
Eiter, T.\BBACOMMA\  \BBA\ Fink, M. \BBOP2003\BBCP.
\newblock \BBOQ Uniform equivalence of logic programs under the stable model
  semantics\BBCQ\
\newblock In {\Bem Proceedings of the 2003 International Conference on Logic
  Programming}, \lowercase{\BVOL}\ 2916 of {\Bem {Lecture Notes in Computer
  Science}}, \BPGS\ 224--238, Berlin. Springer.

\bibitem[\protect\BCAY{Erdem\ \BBA\ Lifschitz}{Erdem\ \BBA\
  Lifschitz}{2003}]{el03}
Erdem, E.\BBACOMMA\  \BBA\ Lifschitz, V. \BBOP2003\BBCP.
\newblock \BBOQ Tight logic programs\BBCQ\
\newblock {\Bem Theory and Practice of Logic Programming}, {\Bem 3\/}(4-5),
  499--518.

\bibitem[\protect\BCAY{Faber, Leone,\ \BBA\ Pfeifer}{Faber
  et~al.}{2004}]{flp04}
Faber, W., Leone, N., \BBA\ Pfeifer, G. \BBOP2004\BBCP.
\newblock \BBOQ Recursive aggregates in disjunctive logic programs: Semantics
  and complexity.\BBCQ\
\newblock In {\Bem {Proceedings of the 9th European Conference on Artificial
  Intelligence (JELIA 2004)}}, \lowercase{\BVOL}\ 3229 of {\Bem LNAI}, \BPGS\
  200 -- 212. Springer.

\bibitem[\protect\BCAY{Fages}{Fages}{1994}]{fag94}
Fages, F. \BBOP1994\BBCP.
\newblock \BBOQ Consistency of {C}lark's completion and existence of stable
  models\BBCQ\
\newblock {\Bem Journal of Methods of Logic in Computer Science}, {\Bem 1},
  51--60.

\bibitem[\protect\BCAY{Ferraris}{Ferraris}{2005}]{fer05}
Ferraris, P. \BBOP2005\BBCP.
\newblock \BBOQ Answer sets for propositional theories\BBCQ\
\newblock In {\Bem Logic Programming and Nonmonotonic Reasoning, 8th
  International Conference, LPNMR 2005}, \lowercase{\BVOL}\ 3662 of {\Bem
  LNAI}, \BPGS\ 119--131. Springer.

\bibitem[\protect\BCAY{Ferraris\ \BBA\ Lifschitz}{Ferraris\ \BBA\
  Lifschitz}{2004}]{fl04}
Ferraris, P.\BBACOMMA\  \BBA\ Lifschitz, V. \BBOP2004\BBCP.
\newblock \BBOQ Weight constraints ans nested expressions\BBCQ\
\newblock {\Bem Theory and Practice of Logic Programming}, {\Bem 5}, 45--74.

\bibitem[\protect\BCAY{Gelfond\ \BBA\ Leone}{Gelfond\ \BBA\
  Leone}{2002}]{GelLeo02}
Gelfond, M.\BBACOMMA\  \BBA\ Leone, N. \BBOP2002\BBCP.
\newblock \BBOQ Logic programming and knowledge representation -- the
  {A}-prolog perspective\BBCQ\
\newblock {\Bem Artificial Intelligence}, {\Bem 138}, 3--38.

\bibitem[\protect\BCAY{Gelfond\ \BBA\ Lifschitz}{Gelfond\ \BBA\
  Lifschitz}{1988}]{gl88}
Gelfond, M.\BBACOMMA\  \BBA\ Lifschitz, V. \BBOP1988\BBCP.
\newblock \BBOQ The stable semantics for logic programs\BBCQ\
\newblock In {\Bem Proceedings of the 5th {I}nternational {C}onference on
  {L}ogic {P}rogramming}, \BPGS\ 1070--1080. MIT Press.

\bibitem[\protect\BCAY{Gelfond\ \BBA\ Lifschitz}{Gelfond\ \BBA\
  Lifschitz}{1991}]{gl90b}
Gelfond, M.\BBACOMMA\  \BBA\ Lifschitz, V. \BBOP1991\BBCP.
\newblock \BBOQ Classical negation in logic programs and disjunctive
  databases\BBCQ\
\newblock {\Bem New Generation Computing}, {\Bem 9}, 365--385.

\bibitem[\protect\BCAY{Hoos\ \BBA\ St{\"u}tzle}{Hoos\ \BBA\
  St{\"u}tzle}{2005}]{hs05}
Hoos, H.\BBACOMMA\  \BBA\ St{\"u}tzle, T. \BBOP2005\BBCP.
\newblock {\Bem Stochastic Local Search Algorithms --- Foundations and
  Applications}.
\newblock Morgan-Kaufmann.

\bibitem[\protect\BCAY{Kemp\ \BBA\ Stuckey}{Kemp\ \BBA\ Stuckey}{1991}]{ks91}
Kemp, D.\BBACOMMA\  \BBA\ Stuckey, P. \BBOP1991\BBCP.
\newblock \BBOQ Semantics of logic programs with aggregates\BBCQ\
\newblock In {\Bem Logic Programming, Proceedings of the 1991 International
  Symposium}, \BPGS\ 387--401. MIT Press.

\bibitem[\protect\BCAY{Lifschitz, Pearce,\ \BBA\ Valverde}{Lifschitz
  et~al.}{2001}]{lpv01}
Lifschitz, V., Pearce, D., \BBA\ Valverde, A. \BBOP2001\BBCP.
\newblock \BBOQ Strongly equivalent logic programs\BBCQ\
\newblock {\Bem ACM Transactions on Computational Logic}, {\Bem 2(4)},
  526--541.

\bibitem[\protect\BCAY{Lin}{Lin}{2002}]{lin02}
Lin, F. \BBOP2002\BBCP.
\newblock \BBOQ Reducing strong equivalence of logic programs to entailment in
  classical propositional logic\BBCQ\
\newblock In {\Bem Principles of Knowledge Representation and Reasoning,
  Proceedings of the 8th International Conference (KR2002)}. Morgan Kaufmann
  Publishers.

\bibitem[\protect\BCAY{Lin\ \BBA\ Zhao}{Lin\ \BBA\ Zhao}{2002}]{lz02}
Lin, F.\BBACOMMA\  \BBA\ Zhao, Y. \BBOP2002\BBCP.
\newblock \BBOQ {ASSAT}: Computing answer sets of a logic program by {SAT}
  solvers\BBCQ\
\newblock In {\Bem Proceedings of the 18th National Conference on Artificial
  Intelligence (AAAI-2002)}, \BPGS\ 112--117. {AAAI Press}.

\bibitem[\protect\BCAY{Liu\ \BBA\ Truszczy{\'n}ski}{Liu\ \BBA\
  Truszczy{\'n}ski}{2003}]{lt03}
Liu, L.\BBACOMMA\  \BBA\ Truszczy{\'n}ski, M. \BBOP2003\BBCP.
\newblock \BBOQ Local-search techniques in propositional logic extended with
  cardinality atoms\BBCQ\
\newblock In Rossi, F.\BED, {\Bem Proceedings of the 9th International
  Conference on Principles and Practice of Constraint Programming, CP-2003},
  \lowercase{\BVOL}\ 2833 of {\Bem {LNCS}}, \BPGS\ 495--509. Springer.

\bibitem[\protect\BCAY{Liu\ \BBA\ Truszczy{\'n}ski}{Liu\ \BBA\
  Truszczy{\'n}ski}{2005a}]{lt05b}
Liu, L.\BBACOMMA\  \BBA\ Truszczy{\'n}ski, M. \BBOP2005a\BBCP.
\newblock \BBOQ Pbmodels - software to compute stable models by pseudoboolean
  solvers\BBCQ\
\newblock In {\Bem Logic Programming and Nonmonotonic Reasoning, Proceedings of
  the 8th International Conference (LPNMR-05)}, {LNAI 3662}, \BPGS\ 410--415.
  Springer.

\bibitem[\protect\BCAY{Liu\ \BBA\ Truszczy{\'n}ski}{Liu\ \BBA\
  Truszczy{\'n}ski}{2005b}]{lt05}
Liu, L.\BBACOMMA\  \BBA\ Truszczy{\'n}ski, M. \BBOP2005b\BBCP.
\newblock \BBOQ Properties of programs with monotone and convex
  constraints\BBCQ\
\newblock In {\Bem Proceedings of the 20th National Conference on Artificial
  Intelligence (AAAI-05)}, \BPGS\ 701--706. AAAI Press.

\bibitem[\protect\BCAY{Manquinho\ \BBA\ Roussel}{Manquinho\ \BBA\
  Roussel}{2005}]{pbcomp05}
Manquinho, V.\BBACOMMA\  \BBA\ Roussel, O. \BBOP2005\BBCP.
\newblock \BBOQ Pseudo boolean evaluation 2005\BBCQ.
\newblock \url{http://www.cril.univ-artois.fr/PB05/}.

\bibitem[\protect\BCAY{Marek, Niemel{\"a},\ \BBA\ Truszczy{\'n}ski}{Marek
  et~al.}{2004}]{mnt03}
Marek, V., Niemel{\"a}, I., \BBA\ Truszczy{\'n}ski, M. \BBOP2004\BBCP.
\newblock \BBOQ Characterizing stable models of logic programs with cardinality
  constraints\BBCQ\
\newblock In {\Bem {Proceedings of the 7th International Conference on Logic
  Programming and Nonmonotonic Reasoning}}, \lowercase{\BVOL}\ 2923 of {\Bem
  Lecture Notes in Artificial Intelligence}, \BPGS\ 154--166. Springer.

\bibitem[\protect\BCAY{Marek, Niemel{\"a},\ \BBA\ Truszczy{\'n}ski}{Marek
  et~al.}{2006}]{mnt06}
Marek, V., Niemel{\"a}, I., \BBA\ Truszczy{\'n}ski, M. \BBOP2006\BBCP.
\newblock \BBOQ Logic programs with monotone abstract constraint atoms\BBCQ\
\newblock {\Bem Theory and Practice of Logic Programming}.
\newblock Submitted.

\bibitem[\protect\BCAY{Marek\ \BBA\ Truszczy\'{n}ski}{Marek\ \BBA\
  Truszczy\'{n}ski}{2004}]{mt04}
Marek, V.\BBACOMMA\  \BBA\ Truszczy\'{n}ski, M. \BBOP2004\BBCP.
\newblock \BBOQ Logic programs with abstract constraint atoms\BBCQ\
\newblock In {\Bem Proceedings of the 19th National Conference on Artificial
  Intelligence (AAAI-04)}, \BPGS\ 86--91. AAAI Press.

\bibitem[\protect\BCAY{Mumick, Pirahesh,\ \BBA\ Ramakrishnan}{Mumick
  et~al.}{1990}]{mum90}
Mumick, I., Pirahesh, H., \BBA\ Ramakrishnan, R. \BBOP1990\BBCP.
\newblock \BBOQ The magic of duplicates and aggregates\BBCQ\
\newblock In {\Bem Proceedings of the 16th International Conference on Very
  Large Data Bases, VLDB 1990}, \BPGS\ 264--277. Morgan Kaufmann.

\bibitem[\protect\BCAY{Niemel{\"a}, Simons,\ \BBA\ Soininen}{Niemel{\"a}
  et~al.}{1999}]{nss99}
Niemel{\"a}, I., Simons, P., \BBA\ Soininen, T. \BBOP1999\BBCP.
\newblock \BBOQ Stable model semantics of weight constraint rules\BBCQ\
\newblock In {\Bem Proceedings of LPNMR-1999}, \lowercase{\BVOL}\ 1730 of {\Bem
  LNAI}, \BPGS\ 317--331. Springer.

\bibitem[\protect\BCAY{Pelov.}{Pelov.}{2004}]{p04}
Pelov., N. \BBOP2004\BBCP.
\newblock \BBOQ Semantics of logic programs with aggregates\BBCQ\
\newblock {\Bem PhD Thesis. Department of Computer Science, K.U.Leuven, Leuven,
  Belgium}.

\bibitem[\protect\BCAY{Pelov, Denecker,\ \BBA\ Bruy\-nooghe}{Pelov
  et~al.}{2004}]{pdb04}
Pelov, N., Denecker, M., \BBA\ Bruy\-nooghe, M. \BBOP2004\BBCP.
\newblock \BBOQ Partial stable models for logic programs with aggregates\BBCQ\
\newblock In Lifschitz, V.\BBACOMMA\  \BBA\ {Niemel\"a}, I.\BEDS, {\Bem Logic
  programming and Nonmonotonic Reasoning, Proceedings of the $7^{th}$
  International Conference}, \lowercase{\BVOL}\ 2923, \BPGS\ 207--219.
  Springer.

\bibitem[\protect\BCAY{Pelov, Denecker,\ \BBA\ Bruynooghe}{Pelov
  et~al.}{2006}]{pdb06}
Pelov, N., Denecker, M., \BBA\ Bruynooghe, M. \BBOP2006\BBCP.
\newblock \BBOQ Well-founded and stable semantics of logic programs with
  aggregates\BBCQ\
\newblock {\Bem Theory and Practice of Logic Programming}.
\newblock Accepted (available at
  \url{http://www.cs.kuleuven.ac.be/~dtai/projects/ALP/TPLP/}).

\bibitem[\protect\BCAY{Selman, Kautz,\ \BBA\ Cohen}{Selman
  et~al.}{1994}]{skc94}
Selman, B., Kautz, H., \BBA\ Cohen, B. \BBOP1994\BBCP.
\newblock \BBOQ Noise strategies for improving local search\BBCQ\
\newblock In {\Bem Proceedings of the 12th National Conference on Artificial
  Intelligence (AAAI-1994)}, \BPGS\ 337--343, Seattle, USA. AAAI Press.

\bibitem[\protect\BCAY{Simons, Niemel{\"a},\ \BBA\ Soininen}{Simons
  et~al.}{2002}]{sns02}
Simons, P., Niemel{\"a}, I., \BBA\ Soininen, T. \BBOP2002\BBCP.
\newblock \BBOQ Extending and implementing the stable model semantics\BBCQ\
\newblock {\Bem Artificial Intelligence}, {\Bem 138}, 181--234.

\bibitem[\protect\BCAY{{Son}\ \BBA\ {Pontelli}}{{Son}\ \BBA\
  {Pontelli}}{2006}]{sp06}
{Son}, T.\BBACOMMA\  \BBA\ {Pontelli}, E. \BBOP2006\BBCP.
\newblock \BBOQ {A constructive semantic characterization of aggregates in
  anser set programming}\BBCQ\
\newblock {\Bem Theory and Practice of Logic Programming}.
\newblock Accepted (available at \url{http://arxiv.org/abs/cs.AI/0601051}).

\bibitem[\protect\BCAY{Son, Pontelli,\ \BBA\ Tu}{Son et~al.}{2006}]{spt06}
Son, T., Pontelli, E., \BBA\ Tu, P. \BBOP2006\BBCP.
\newblock \BBOQ Answer sets for logic programs with arbitrary abstract
  constraint atoms\BBCQ\
\newblock In {\Bem Proceedings of the 21st National Conference on Artificial
  Intelligence (AAAI-06)}. AAAI Press.

\bibitem[\protect\BCAY{Syrj{\"a}nen}{Syrj{\"a}nen}{1999}]{syr99a}
Syrj{\"a}nen, T. \BBOP1999\BBCP.
\newblock \BBOQ {\em lparse}, a procedure for grounding domain restricted logic
  programs\BBCQ\
\newblock \url{http://www.tcs.hut.fi/Software/smodels/lparse/}.

\bibitem[\protect\BCAY{Turner}{Turner}{2003}]{tu03}
Turner, H. \BBOP2003\BBCP.
\newblock \BBOQ Strong equivalence made easy: Nested expressions and weight
  constraints\BBCQ\
\newblock {\Bem Theory and Practice of Logic Programming}, {\Bem 3, (4\&5)},
  609--622.

\bibitem[\protect\BCAY{Walser}{Walser}{1997}]{wal97}
Walser, J. \BBOP1997\BBCP.
\newblock \BBOQ Solving linear pseudo-boolean constraints with local
  search\BBCQ\
\newblock In {\Bem Proceedings of the 14th National Conference on Artificial
  Intelligence (AAAI-97)}, \BPGS\ 269--274. AAAI Press.

\end{thebibliography}

}

\end{document}